\documentclass{article}

\usepackage[utf8]{inputenc} % allow utf-8 input
\usepackage[T1]{fontenc}    % use 8-bit T1 fonts
\usepackage{hyperref}       % hyperlinks
\usepackage{url}            % simple URL typesetting
\usepackage{booktabs}       % professional-quality tables
\usepackage{amsfonts}       % blackboard math symbols
\usepackage{nicefrac}       % compact symbols for 1/2, etc.
\usepackage{microtype}      % microtypography
\usepackage{enumerate}

\usepackage[round]{natbib}

\usepackage{bm}
\usepackage{boldline} 
\usepackage{colortbl} 
\usepackage{multirow}
\usepackage{etoc}
\usepackage{color}
\usepackage{lipsum}
\usepackage[table,svgnames]{xcolor}
\usepackage{amsmath}
\usepackage{amssymb}
\usepackage{mathtools}
\usepackage{amsthm}
\usepackage{booktabs}       
\usepackage{algorithm}
\usepackage{algorithmic}
\usepackage{subcaption}
\usepackage{url}
\DeclareMathOperator*{\argmin}{arg\,min}

\newtheorem{proposition}{Proposition}
\newtheorem{theorem}{Theorem}

\newtheorem{corollary}{Corollary}
\newtheorem{lemma}{Lemma}

\newtheorem{assumption}{Assumption}
\newtheorem{definition}{Definition}

\title{Optimistic Safety for Online Convex Optimization with Unknown Linear Constraints}
\date{}
\author{Spencer Hutchinson\footnotemark[1] \and Tianyi Chen\footnotemark[2] \and Mahnoosh Alizadeh\footnotemark[1]\\
\and
\footnotemark[1] University of California, Santa Barbara\\
\footnotemark[2] Rensselaer Polytechnic Institute
}

\newcommand{\Fc}{\mathcal{F}}

\newcommand{\Xc}{\mathcal{X}}

\newcommand{\Rb}{\mathbb{R}}
\newcommand{\Eb}{\mathbb{E}}
\newcommand{\Ac}{\mathcal{A}}

\newcommand{\Nb}{\mathbb{N}}

\newcommand{\xtil}{\tilde{x}}

\newcommand{\Econf}{\mathcal{E}_{\mathrm{conf}}}

\newcommand{\Bb}{\mathbb{B}}

\newcommand{\Pb}{\mathbb{P}}

\newcommand{\Yc}{\mathcal{Y}}
\newcommand{\Yctil}{\tilde{\mathcal{Y}}}

\newcommand{\Oc}{\mathcal{O}}

\newcommand{\Nc}{\mathcal{N}}

\newcommand{\Octil}{\tilde{\mathcal{O}}}
\newcommand{\Hc}{\mathcal{H}}

\newcommand{\Ec}{\mathcal{E}}
\newcommand{\Ib}{\mathbb{I}}
\newcommand{\Eazuma}{\mathcal{E}_{\mathrm{azuma}}}
\newcommand{\bmin}{b_{\mathrm{min}}}

\newcommand{\tone}{\mathrm{Term\ I}}
\newcommand{\ttwo}{\mathrm{Term\ II}}

\newcommand{\tthree}{\mathrm{Term\ III}}
\newcommand{\cosa}{\textrm{Term~I}}
\newcommand{\opr}{\textrm{Term~II}}

\newcommand{\alglabel}{%
  \addtocounter{ALC@line}{-1}% Reduce line counter by 1
  \refstepcounter{ALC@line}% Increment line counter with reference capability
  \label% Regular \label
}

\allowdisplaybreaks

\begin{document}

\maketitle

\begin{abstract}
% \lipsum[1]
    We study the problem of online convex optimization (OCO) under unknown linear constraints that are either static, or stochastically time-varying.
    For this problem, we introduce an algorithm that we term \emph{Optimistically Safe OCO} (\mbox{OSOCO}) and show that it enjoys $\Octil(\sqrt{T})$ regret and \emph{no constraint violation}.
    In the case of static linear constraints, this improves on the previous best known $\tilde{\mathcal{O}}(T^{2/3})$ regret under the same assumptions.
    In the case of stochastic time-varying constraints, our work supplements existing results that show $\Oc(\sqrt{T})$ regret and $\Oc(\sqrt{T})$ cumulative violation under more general convex constraints and a different set of assumptions.
    In addition to our theoretical guarantees, we also give numerical results that further validate the effectiveness of our approach.
\end{abstract}

\section{Introduction}

Online convex optimization (OCO), formalized by \cite{zinkevich2003online}, is a setting where, in each round $t \in [T]$, a player chooses a vector action $x_t$, suffers the cost of that action according to an adversarially-chosen cost function $f_t(x_t)$, and subsequently observes the cost function $f_t$.
This setting has found broad applicability in various fields including online advertising \citep{mcmahan2013ad}, network routing \citep{awerbuch2008online}, portfolio management \citep{cover1991universal} and resource allocation \citep{yu2019learning}.
However, despite its significance, this fundamental setting assumes that any constraints on the player's actions are known a priori, which is not the case in many real-world applications.
As a result, the problem of OCO under unknown constraints has received considerable attention in the last decade, e.g. \citep{mannor2009online,yu2017online,cao2018online,guo2022online}.

Within the broader literature on OCO with unknown constraints, this paper focuses on two important research directions: (1)
\emph{OCO with stochastic constraints}, where there are different constraints in each round $g_t(x_t) \leq 0 \ \forall t$ that are independent and identically-distributed (i.i.d.) and the player receives feedback on $g_t$ after choosing each action $x_t$, and (2) \emph{OCO with static constraints}, where there is a static and deterministic constraint $g(x_t) \leq 0 \ \forall t$ for which the player receives noisy feedback $g(x_t) + \epsilon_t$ after choosing each action $x_t$.
In OCO with stochastic constraints, existing works typically leverage primal-dual algorithms to ensure both sublinear regret and sublinear cumulative violation $\sum_{t=1}^{T} g_t(x_t)$.
These algorithms operate by  jointly iterating over both the actions (primal variables) and constraint penalties (dual variables), and have been shown to ensure $\Oc(\sqrt{T})$\footnote{For ease of presentation in the introduction, we use $\Oc(\cdot)$ to hide non-$T$ quantities, and $\Octil(\cdot)$ to further hide $log$ factors.} expected regret and $\Oc(\sqrt{T})$ expected cumulative violation  \citep{yu2017online,wei2020online,castiglioni2022unifying,yu2023online}.
On the other hand, the works on OCO with static constraints ensure \emph{no violation} $g(x_t) \leq 0\ \forall t$ with high probability, by explicitly estimating the constraints and then only choosing actions that are verifiably-safe \citep{chaudhary2022safe,chang2023dynamic}.
However, the approach in \cite{chaudhary2022safe} and \cite{chang2023dynamic} requires a pure-exploration phase of
$\Theta(T^{2/3})$ rounds and therefore incurs a less favorable regret of $\Octil(T^{2/3})$.

In this work, we introduce an algorithm for both static and stochastic constraints that \emph{directly estimates the constraints,} and therefore ensures no violation, while avoiding the costly exploration phase used in prior works \citep{chaudhary2022safe,chang2023dynamic}.
Precisely, our algorithm uses \emph{optimistic} estimates of the constraints (outer-approximations of the feasible set) to identify actions that are low-regret, but potentially unsafe, and then \emph{pessimistic} estimates of the constraints (inner-approximations of the feasible set) to find the scalings on these actions that will ensure constraint satisfaction.
This approach, which we call \emph{optimistic safety}, allows for the constraint to be learned \emph{simultaneous} to the regret minimization and therefore does not require a pure-exploration phase as in prior works.
Ultimately, our approach yields the following results:
\begin{itemize}
    \item In OCO with static linear constraints and noisy constraint feedback, we show that our algorithm enjoys $\Octil(\sqrt{T})$ regret and no violation, in either high probability (Theorem~\ref{thm:main_reg}) or in expectation (Proposition~\ref{prop:expec_viol}).
    Our high probability guarantees improve on the $\Octil(T^{2/3})$ regret attained by prior work on this problem, under the same assumptions and violation guarantees \citep{chaudhary2022safe,chang2023dynamic}.
    \item We then extend our results to OCO with stochastic linear constraints $g_t$, where the player receives bandit feedback on the constraint $g_t(x_t)$ (Section~\ref{sec:tvar}). We show $\Octil(\sqrt{T})$ expected regret and no violation in expectation $\Eb[g_t(x_t)] \leq 0\ \forall t$ in this setting (Corollary~\ref{cor:tvar}).
    This supplements existing results on OCO with stochastic convex constraints that show $\Oc(\sqrt{T})$ expected regret and $\Oc(\sqrt{T})$ expected cumulative violation $\Eb[\sum_{t} g_t(x_t)]$, although with a different set of assumptions (see Section \ref{sec:stoch_prob}) \citep{yu2017online,wei2020online,castiglioni2022unifying,yu2023online}.
\end{itemize}

\subsection{Overview}

In Section \ref{sec:prob_set}, we give the problem setup and assumptions.
Then, we design a meta-algorithm in Section \ref{sec:alg} that leverages the idea of optimistic safety.
In Section \ref{sec:anal}, we show that this meta-algorithm ensures $\Octil(d \sqrt{T})$ regret and no constraint violation in either high-probability or expectation.
In Section \ref{sec:comp_eff}, we improve the computational efficiency of our original algorithm by giving a version that only requires $2 d$ convex projections in each round, although with a slightly larger dependence on the dimension in the regret bound, i.e. $\Octil(d^{3/2}\sqrt{T})$.
In Section \ref{sec:tvar}, we extend our approach to the setting of OCO with stochastic linear constraints and bandit constraint feedback and show matching regret bounds in expectation and no expected violation.
Lastly, we give simulation results in Section \ref{sec:num_exp} that showcase the empirical performance of our algorithm.

\subsection{Related work}

In the following, we discuss prior work on OCO with unknown static constraints, OCO with time-varying constraints, and safe learning.
Lastly, we discuss a concurrent work.
In Appendix~\ref{sec:more_rel_work}, we provide additional discussion on OCO with long-term constraints, projection-free OCO, and online control.

\paragraph{OCO with Unknown Static Constraints}
The problem of OCO with unknown linear constraints and noisy constraint feedback was first studied by \cite{chaudhary2022safe} who gave an algorithm with $\Octil(d T^{2/3})$ regret and no violation $g(x_t) \leq 0\ \forall t$, with high probability.
In the same setting, we give an algorithm that enjoys $\Octil(d \sqrt{T})$ regret and no violations with high probability, and a more efficient variant that enjoys $\Octil(d^{3/2} \sqrt{T})$ regret. With a different choice of algorithm parameters, we also show matching bounds on the expected regret and show no expected violation $\Eb[g_t(x_t)] \leq 0$.
The approach taken by \cite{chaudhary2022safe} was then extended to the distributed and nonconvex setting in \cite{chang2023dynamic}, where an $\Octil(d T^{2/3} + T^{1/3} C_T)$ high-probability regret bound was given with $C_T$ as the path length of the best minimizer sequence.
Another variant of OCO with unknown static constraints was considered by \cite{hutchinson2024safe}, although they assume that the constraint function is strongly-convex, and thus their setting is distinct from the one we consider here.

\paragraph{OCO with Time-varying Constraints}

For OCO with time-varying constraints that are chosen adversarially, \cite{mannor2009online} showed that no algorithm can simultaneously achieve both sublinear cumulative violation and sublinear regret against the best action that satisfies the constraints on average.
Following this negative result, considerable effort has been focused on ways to weaken the setting to guarantee sublinear regret and cumulative violation, e.g. \citep{neely2017online,sun2017safety,chen2017online,liakopoulos2019cautious, cao2018online,yi2022regret,guo2022online,castiglioni2022unifying,kolev2023online}.
Most relevant to this paper, several works study the setting in which the constraints are convex functions that are stochastic and i.i.d., and give algorithms with $\Oc(\sqrt{T})$ expected regret and $\Oc(\sqrt{T})$ expected cumulative violation when the player receives full feedback \citep{yu2017online,wei2020online,castiglioni2022unifying} and $\Oc(d \sqrt{T})$ expected regret and violation when the player receives two-point bandit feedback \citep{yu2023online}.
In the present work, we supplement this literature by consider a setting (in Section \ref{sec:tvar}) where the constraints are i.i.d. linear functions and the player receives one-point bandit feedback on the constraint (but full feedback on the cost).
In this setting, we guarantee $\Octil(d \sqrt{T})$ expected regret (or $\Octil(d^{3/2} \sqrt{T})$ for our more efficient version) and \emph{no expected violation}.
We additionally assume that the player knows an action that is strictly feasible in expectation and its expected constraint value, as discussed in Section \ref{sec:stoch_prob}.

\paragraph{Safe Learning}

Our setting falls under the classification of ``safe learning,'' which often deals with learning problems that have uncertain constraints.
Some examples of such problems are constrained MDPs e.g. \citep{achiam2017constrained,wachi2020safe,liu2021learning,amani2021safe,ghosh2024towards}, safe Gaussian process bandits e.g. \citep{sui2015safe,sui2018stagewise,losalka2024no}, safe optimization with uncertain constraints e.g. \citep{usmanova2019safe,fereydounian2020safe} and safe stochastic linear bandits e.g. \citep{amani2019linear,moradipari2021safe,pacchiano2021stochastic,hutchinson2024directional}.
The most relevant to the present work is the problem of safe stochastic linear bandits \citep{moradipari2021safe,pacchiano2021stochastic,hutchinson2024directional} as it has the same constraint formulation as OCO with static linear constraints and both concern regret minimization.
However, the safe stochastic linear bandit setting has stochastic linear costs with bandit feedback, while OCO with static linear constraints has adversarial convex costs with full cost feedback.
Note that our approach draws inspiration from the safe stochastic linear bandit work \cite{hutchinson2024directional}, although we emphasize that OCO with static linear constraints has fundamental challenges that are not present in the safe stochastic linear bandit problem.
Most notably, OCO with static linear constraints has the challenges posed by time-varying and nonconvex action sets in OCO, for which we use a phased-based approach and a novel OCO algorithm for nonconvex action sets.

\subsection{Notation}
\label{sec:nots}

We use $\Oc(\cdot)$ to refer to big-O notation and $\Octil(\cdot)$ to refer to the same ignoring log factors.
Also, we denote the 2-norm by $\| \cdot \|$.
For a natural number $n$, we define $[n] := \{ 1,2,...,n\}$.
For a matrix $M$, the transpose of $M$ is denoted by $M^\top$.
We also use $\mathbf{1}$ and $\mathbf{0}$ to refer to a vector of ones and zeros, respectively.
Given $p \in [1,\infty]$, we denote the $p$-norm ball as $\Bb_p$ (where $\Bb$ is the $2$-norm ball).
For vectors $x,y \in \Rb^d$, we use $x \leq y$ to refer to the component-wise partial ordering, i.e. $x \leq y\ \iff x_i \leq y_i \ \forall i\in [d]$. 
Lastly, we use $\Eb_t$ to refer to the expectation conditioned on the randomness up to and including round~$t - 1$.

\section{Problem Setup}
\label{sec:prob_set}

The problem is played over $T$ rounds, where in each round $t \in [T]$, a player chooses an action $x_t$ from the convex action set $\Xc \subseteq \Rb^d$ and, simultaneously, an adversary chooses the convex cost function $f_t : \Xc \rightarrow \Rb$ according to the actions chosen by the player in previous rounds.
The player then suffers the cost $f_t(x_t)$, observes the cost function $f_t$ and receives noisy constraint feedback $y_t := A x_t + \epsilon_t$.
The matrix $A \in \Rb^{n \times d}$ is the unknown and deterministic constraint matrix, and the vector $\epsilon_t \in \Rb^n$ is a zero-mean random noise.
Additionally, there is a deterministic constraint limit $b \in \Rb^n$ that is known to the player.\footnote{The assumption that $b$ is known is also used in prior work on this problem \citep{chaudhary2022safe,chang2023dynamic}.}

We aim either to ensure that there is \emph{no constraint violation in high-probability} or \emph{no constraint violation in expectation}.
These two types of guarantees are summarized as follows:
\begin{itemize}
    \item[--] \textbf{High-probability:} $A x_t \leq b$ for all $t \in [T]$, with high probability.
    \item[--] \textbf{Expectation:} $\Eb[A x_t] \leq b$ for all $t \in [T]$.
\end{itemize}

In either case, we additionally aim to minimize the regret with respect to the best constraint-satisfying action in hindsight, i.e.
\begin{equation*}
    R_T := \sum_{t=1}^T f_t(x_t) - \sum_{t=1}^T f_t(x^\star),
\end{equation*}
where $x^\star \in \argmin_{x \in \Yc} \sum_{t=1}^T f_t(x)$ and $\Yc = \{ x \in \Xc : A x \leq b \}$.

Our approach to this setting relies on the following assumptions.
We first assume that the costs have bounded gradients and the action set is compact, which are standard in OCO (e.g. \cite{zinkevich2003online,hazan2016introduction}).

\begin{assumption}[Bounded gradients]
\label{ass:cost_funcs}
    For all $t \in [T]$, $f_t$ is differentiable and $\| \nabla f_t (x) \| \leq G$ for all $x \in \Xc$.
\end{assumption}

\begin{assumption}[Compact action set]
\label{ass:set_bound}
    The action set $\Xc$ is closed and $\| x \| \leq D/2$ for all $x \in \Xc$.
\end{assumption}

We further assume that elements of $b$ are positive, the noise is conditionally sub-gaussian, and that the constraint parameter $A$ is bounded.
These match the assumptions used in \cite{chaudhary2022safe} and \cite{chang2023dynamic}.

\begin{assumption}[Initial safe action]
\label{ass:init}
    It holds that $\mathbf{0} \in \Xc$ and $\bmin := \min_{i \in [n]} b_i > 0$ where $b_i$ is the $i$th element of $b$.\footnote{If a nonzero safe action $x^s$ and its constraint value $A x^s$ are given (as in \cite{chaudhary2022safe}), then the problem can be shifted such that the safe action is at the origin. This will not impact the other problem parameters, except that $D$ may be increased by a factor $\leq 2$.}
\end{assumption}

\begin{assumption}[Conditionally-subgaussian noise]
\label{ass:noise}
    The noise terms $(\epsilon_t)_{t \in [T]}$ are element-wise conditionally $\rho$-subgaussian, i.e. it holds for all $t \in [T]$ and $i \in [n]$ that $\Eb[\epsilon_{t,i} | x_1, \epsilon_1,...,\epsilon_{t-1}, x_t] = 0$ and $\Eb[\exp(\lambda \epsilon_{t,i}) | x_1, \epsilon_1,...,\epsilon_{t-1}, x_t] \leq \exp(\frac{\lambda^2 \rho^2}{2}), \forall \lambda \in \Rb$ where $\epsilon_{t,i}$ is the $i$th element of $\epsilon_t$.
\end{assumption}

\begin{assumption}[Bounded constraint]
\label{ass:const}
    Each row of $A$ is bounded, i.e. $\| a_i \| \leq S$ for all $i \in [n]$ where $a_i^\top$ is the $i$th row of $A$.
\end{assumption}

\section{Algorithm}

\begin{algorithm}[t]
    \caption{Optimistically Safe OCO (OSOCO)}
    \label{alg:main_alg}
\begin{algorithmic}[1]
    \INPUT Action set $\Xc$, Horizon $T$, Online algorithm $\Ac$, Confidence set radius $\beta_t > 0$, Regularization $\lambda > 0$, Tightening $\kappa \in [0,\bmin)$.
    \STATE Initialize: $t = 1, j = 1, V_t = \lambda I, S_t = \mathbf{0}$. 
    \WHILE{$t \leq T$}
        \STATE Initialize phase: $\bar{\beta}_j = \beta_t, \bar{S}_j = S_t, \bar{V}_j = V_t$.\alglabel{lne:start_init}
        \STATE Estimate constraints: $\hat{A}_j = \bar{S}_j \bar{V}_j^{-1}$.\alglabel{lne:est}
        \STATE Update pessimistic action set:\\  $\Yc_j^p := \left\{ x \in \Xc : \hat{A}_j x + \bar{\beta}_j \| x \|_{\bar{V}_j^{-1}} \mathbf{1} \leq b - \kappa \mathbf{1} \right\}$\alglabel{lne:pess_set}
        \STATE Update optimistic action set:\\  $\Yc_j^o := \left\{ x \in \Xc : \hat{A}_j x - \bar{\beta}_j \| x \|_{\bar{V}_j^{-1}} \mathbf{1}  \leq b - \kappa \mathbf{1} \right\}$\alglabel{lne:opt_set}
        \STATE Initialize $\Ac$ with $\Yc_j^o$ as action set.\alglabel{lne:start_hd}
        \WHILE{$\det(V_t) \leq 2 \det(\bar{V}_j)$ \textbf{and} $t \leq T$ } \alglabel{lne:while} 
            \STATE Receive $\xtil_t$ from $\Ac$.\alglabel{lne:opt_act}
            \STATE Find safe scaling:\\ $\gamma_t = \max \left\{\mu \in \left[ 0, 1 \right] : \mu \xtil_t \in \Yc_j^p \right\}$.\alglabel{lne:safe_scale}
            \STATE Play $x_t = \gamma_t \xtil_t $ and observe $f_t, y_t$.\alglabel{lne:play_act}
            \STATE Send $f_t$ to $\Ac$.\alglabel{lne:send_costs}
            \STATE $V_{t+1} = V_t + x_t x_t^\top, S_{t+1} = S_t + y_t x_t^\top$.\alglabel{lne:gram_upd}
            \STATE $t = t + 1$.\alglabel{lne:phase_end}
        \ENDWHILE
        \STATE Terminate $\Ac$. \alglabel{lne:term}
        \STATE $j = j + 1$.
    \ENDWHILE
\end{algorithmic}
\end{algorithm}

\label{sec:alg}

To address the stated problem, we propose the algorithm \emph{Optimistically Safe OCO} (OSOCO) given in Algorithm \ref{alg:main_alg}.
This algorithm operates over phases and invokes a generic online optimization algorithm $\Ac$ in each phase.
We give the details on OSOCO in Section \ref{sec:opt_pess}.
Then, in Section \ref{sec:hedge_desc}, we give the requirements on the online optimization algorithm $\Ac$.

\subsection{Details on OSOCO}
\label{sec:opt_pess}

The OSOCO algorithm operates over a number of phases.
We describe the initialization of each phase, the update in each round, and then the phase termination criteria in the following.

\paragraph{Phase Initialization (lines \ref{lne:start_init}-\ref{lne:start_hd})}
At the beginning of each phase $j$, the regularized least-squares estimator of the constraint $\hat{A}_j$ is calculated (line \ref{lne:est}), which is then used to construct the pessimistic action set $\Yc_j^p$ (line \ref{lne:pess_set}) and the optimistic action set $\Yc_j^o$ (line \ref{lne:opt_set}).
Then, the online optimization algorithm $\Ac$ is initialized with the optimistic action set (line \ref{lne:start_hd}).
To appropriately construct the pessimistic and optimistic action sets, we use bounds on the constraint function of the form,
\begin{equation}
    \label{eqn:const_bounds}
    \hat{A}_j x - \bar{\beta}_j \| x \|_{\bar{V}_j^{-1}} \mathbf{1} \leq A x \leq \hat{A}_j x + \bar{\beta}_j \| x \|_{\bar{V}_j^{-1}} \mathbf{1}
\end{equation}
which were shown in \cite{abbasi2011improved} to hold for all rounds with high probability given an appropriate choice of confidence radius $\bar{\beta}_j$.
It is natural that this bound is proportional to the weighted norm $\| x \|_{\bar{V}_j^{-1}}$ because this results in the bound being tighter when $x$ is in a direction that has been played more frequently and therefore in a direction for which the least-squares estimator is more accurate.
Additionally, we incorporate a tightening parameter $\kappa$ in the optimistic and pessimistic sets which is used to ensure no violation in expectation.
Therefore, the bounds on the constraint function in \eqref{eqn:const_bounds} ensure that the optimistic set (resp. pessimistic set) is a superset (resp. subset) of the tightened feasible set $\tilde{\Yc} = \{ x \in \Xc : Ax \leq b - \kappa \mathbf{1} \}$ for all phases with high probability, i.e. $\Yc_j^p \subseteq \tilde{\Yc} \subseteq \Yc_j^o$ for all $j$ with high probability.

\paragraph{Update in Each Round (lines \ref{lne:opt_act}-\ref{lne:phase_end})}
In each round, the online optimization algorithm $\Ac$ chooses a point $\xtil_t$ from the optimistic set (line \ref{lne:opt_act}), and then the OSOCO algorithm finds the largest scaling $\gamma_t$ on $\xtil_t$ such that it is in the pessimistic set (line \ref{lne:safe_scale}), before playing $x_t = \gamma_t \xtil_t$ (line \ref{lne:play_act}).
The matrices used for estimation ($V_t,S_t$) are then updated (lines \ref{lne:gram_upd}).
The reason that we \emph{scale} the optimistic action $\xtil_t$ down into the pessimistic set rather than using another method, such as the euclidean projection, is because scaling this action down ensures that we are still choosing an action in the \emph{direction} of the optimistic action which gives the algorithm more constraint information in this direction.
This ultimately leads to effective decision-making because the optimistic actions are updated on the optimistic set and therefore will naturally gravitate towards areas that are uncertain and/or areas that have low regret, effectively balancing exploitation and exploration.

\paragraph{Phase Termination}
The phase ends when the determinant of the Gram matrix ($V_{t+1}$) has doubled (line \ref{lne:while}), after which it terminates $\Ac$ (line \ref{lne:term}) and starts a new phase.
This type of ``rare update'' ensures that the phase only ends when a significant amount of new constraint information has been collected.
As such, it results in a small number $\Oc(d \log(T))$ of phases and therefore there will only be a small cost to restarting the OCO algorithm in each phase.
Note that a related form of rare updating was used as early as \cite{abbasi2011improved} to reduce the number of determinant calculations in stochastic linear bandits.

\subsection{Requirements of Online Algorithm $\Ac$}
\label{sec:hedge_desc}

The online optimization algorithm $\Ac$ needs to be chosen from a specific class of algorithms in order for OSOCO to work.
In particular, the optimistic action set is not necessarily convex and $\Ac$ can be terminated at any time, so $\Ac$ needs to perform well when its action set is nonconvex and when the horizon is unknown.
We make this concrete in the following.

\begin{definition}[Regret Bound on $\Ac$]
    \label{def:reg}
    Consider an OCO setting where the action set $\Xc$ is not necessarily convex, and the problem can be terminated at any time (and therefore the final round $T$ is not known to the player).
    Also, we assume Assumptions \ref{ass:cost_funcs} and \ref{ass:set_bound} hold, and that $\Ac$ chooses a distribution $p_t$ in each round from which it samples $x_t$.\footnote{We assume that $\Ac$ is randomized because an online optimization algorithm needs randomness to achieve sublinear regret when the action set is not convex, see \cite{cesa2006prediction}.}
    Then, $C_{\Ac}$ is a positive real such that,
    \begin{equation*}
        \sum_{t \in [T]} \Eb_{x_t \sim p_t} [f_t(x_t)] - \min_{x \in \Xc} \sum_{t \in [T]} f_t(x) \leq C_{\Ac} \sqrt{T},
    \end{equation*}
    if the algorithm $\Ac$ is played in any such online setting. 
    We let $C_{\Ac}$ depend on $T$ at most $\Oc(\mathrm{polylog}(T))$.
\end{definition}

Most algorithms that admit such a regret bound are those designed for the online nonconvex optimization setting as they can handle the nonconvex action set (and then the ``doubling trick'' can be used to handle the unknown horizon).
Some examples of applicable algorithms are FTPL \citep{kalai2005efficient} and continuous versions of the Hedge algorithm \citep{freund1997decision} which enjoy $\Oc(d^{3/2} \sqrt{T})$ and $\Octil(\sqrt{d T})$ regret, respectively.
We explicitly give several example algorithms and their regret guarantees in Appendix \ref{sec:examp_algs}.
The primary drawback of these algorithms is that they are often computationally inefficient.
In Section \ref{sec:comp_eff}, we provide a more computationally efficient version of OSOCO at the cost of a slightly worse regret bound (i.e. $\Octil(d^{3/2} \sqrt{T})$ instead of $\Octil(d \sqrt{T})$).

\section{Analysis}
\label{sec:anal}

In this section, we provide regret and safety guarantees for OSOCO.
The high-probability guarantees and expectation guarantees are given in Sections \ref{sec:static_const} and \ref{sec:time_var}, respectively.

\subsection{High-probability Guarantees}
\label{sec:static_const}

The following theorem shows that, with high probability, OSOCO ensures no violation $A x_t \leq b$ and that the regret is $\Octil((d + C_{\Ac} \sqrt{d}) \sqrt{T})$, where $C_{\Ac}$ is the regret constant of the online algorithm $\Ac$ that is used by OSOCO (see Definition~\ref{def:reg}).
It immediately follows that the regret of OSOCO is $\Octil(d \sqrt{T})$ when $C_{\Ac}$ is $\Octil(\sqrt{d})$, which holds when algorithms such as Hedge \citep{freund1997decision} are used for $\Ac$ (given in Appendix~\ref{sec:examp_algs}).

\begin{theorem}
    \label{thm:main_reg}
    Fix some $\delta \in (0,1/2)$, and set $\lambda = \max(1, D^2)$, $\beta_t = \rho \sqrt{ d \log \left(\frac{1 + (t - 1) D^2/\lambda}{\delta/n} \right)} + \sqrt{\lambda} S$ and $\kappa = 0$.
    Then if $T \geq 3$, with probability at least $1 - 2 \delta$, the actions of OSOCO (Algorithm~\ref{alg:main_alg}) satisfy,
    \begin{equation*}
        \begin{split}
            R_T \leq & \mbox{ } 4 \frac{D G}{\bmin} \beta_T \sqrt{3 d T \log \left( T \right)} + C_{\Ac} \sqrt{4 d T \log(T)}\\
            & + 2 D G \sqrt{2 T \log(1/\delta)} 
        \end{split}
    \end{equation*}
    and $A x_t \leq b$ for all $t \in [T]$.
\end{theorem}

The proof of Theorem \ref{thm:main_reg} (given in Appendix \ref{sec:main_proof}), decomposes the regret as,
\begin{equation*}
    R_T = \underbrace{\sum_{t \in [T]} (f_t(x_t) -  f_t(\xtil_t))}_{\cosa} + \underbrace{\sum_{t \in [T]} (f_t(\xtil_t) - f_t(x^\star)).}_{\opr}
\end{equation*}
Intuitively, \cosa{} is the \emph{cost of safety}, i.e. the cost that the algorithm incurs by maintaining no violation, while \opr{} is the \emph{optimistic regret}, i.e. the regret that the algorithm would incur if it played the optimistic actions $\xtil_t$ (although this would violate the constraints).
We bound each of these in the following lemmas.

\begin{lemma}[Cost of Safety]
    \label{lem:reg_safety}
    Under the same conditions as Theorem \ref{thm:main_reg}, it holds that
    \begin{equation}
        \label{eqn:reg_safety}
        {\textstyle\mathrm{Term\ I} \leq \frac{4 D G }{\bmin} \beta_T \sqrt{3 d T \log \left( T \right)}.}
    \end{equation}
\end{lemma}

The key idea behind the proof of Lemma \ref{lem:reg_safety} is that, if we can relate \cosa{} to the estimation error of the constraint given in \eqref{eqn:const_bounds}, then we can show that it is $\Octil(\sqrt{T})$ by the elliptic potential lemma (\cite{dani2008stochastic}).
One difficulty in doing so is relating the distance between $\xtil_t$ and $x_t$ to the estimation error.
We do so by first establishing a lower bound on the safe scaling $\gamma_t$ (line \ref{lne:safe_scale}) in terms of the estimation error at the optimistic action $\bar{\beta}_j \| \xtil_t \|_{V_j^{-1}}$ and then relating this to the estimation error at the played actions $\bar{\beta}_j \| x_t \|_{V_j^{-1}}$ using the fact that $\xtil_t$ and $x_t$ are scalar multiples of each other.
Then, since all $f_t$ have bounded gradients and thus Lipschitz, we can bound $\cosa$ in terms of $\| \xtil_t - x_t \|$ and therefore in terms of $1 - \gamma_t$ and then $\bar{\beta}_j \| x_t \|_{V_j^{-1}}$.
Finally, we adapt the rare updating analysis from \cite{abbasi2011improved} to account for the fact that the optimistic and pessimistic sets are updated only at the beginning of each phase, and then apply the elliptic potential lemma to get $\Octil(\sqrt{T})$ bound.

\begin{lemma}[Optimistic Regret]
    \label{lem:opt_reg}
    Under the same conditions as Theorem \ref{thm:main_reg}, it holds with probability at least $1 - 2 \delta$ that
    \begin{align*}
        \mathrm{Term\ II} \leq  C_{\Ac} \sqrt{4 d T \log(T)} + D G \sqrt{2 T \log(1/\delta)}.
    \end{align*}
\end{lemma}

The key idea behind the proof of Lemma \ref{lem:opt_reg} is that the optimistic set $\Yc_j^o$ contains the feasible set $\Yc$ with high probability (when $\kappa = 0$), and therefore we can bound \opr{} in terms of the sum of the regrets that $\Ac$ incurs in each phase.
Concretely, with $\tau_j$ denoting the duration of the $j$th phase and $N$ the number of phases, it holds with high probability,\\
\begin{equation}\label{eqn:opt_reg_ske}
    \begin{split}
        & {\textstyle \sum_{t=1}^T \Eb_t [f_t(\xtil_t)] - \sum_{t=1}^T f_t(x^\star)}\\
        & {\textstyle \leq \sum_{j =1}^N C_{\Ac} \sqrt{\tau_j} \leq C_{\Ac} \sqrt{N T} \leq C_{\Ac} \sqrt{4 d T \log(T)}},
    \end{split}
\end{equation}
where the second inequality is Cauchy-Schwarz and the third inequality uses the fact that $N \leq 2 d \log(T)$ thanks to the phase termination criteria.
Lastly, we use the Azuma inequality to convert \eqref{eqn:opt_reg_ske} to a high probability bound on \opr{}.

\subsection{Expectation Guarantees}
\label{sec:time_var}

The following proposition shows that OSOCO ensures no expected violation $\Eb[A x_t] \leq b$ and that the expected regret is $\Octil((d + C_{\Ac} \sqrt{d}) \sqrt{T})$.

\begin{proposition}
    \label{prop:expec_viol}
    Choose $\delta =  \min(1/2,\bmin/(2 S D T))$ and $\kappa = \bmin/T$, and set $\lambda$ and $\beta_t$ the same as in Theorem \ref{thm:main_reg}.
    Then, if $T \geq 3$, it holds that $\Eb[A x_t] \leq b$ for all $t \in [T]$ and
    \begin{equation*}
        \label{eqn:expec_reg}
        \begin{split}
            \Eb[R_T] \leq & \mbox{ } \frac{6 D G}{\bmin} \beta_T \sqrt{2 d T \log \left( T \right)} + C_{\Ac} \sqrt{2 d T \log(T)} \\
            & + 2 D G \sqrt{2 T \log(1/\delta)} + G D + \frac{\bmin G}{S}.
        \end{split}
    \end{equation*}
\end{proposition}

The complete proof of Proposition \ref{prop:expec_viol} is given in Appendix \ref{sec:prf_expec}.
The key idea behind this proof is that, by tightening the constraint, we can cancel out the violation that might occur if the high probability event in Theorem \ref{thm:main_reg} does not hold.
To see this, let $\Ec$ be the high probability event in Theorem \ref{thm:main_reg} and note that for all $i \in [n]$ and $t \in [T]$,
\begin{equation*}{\textstyle
    \begin{aligned}
        \Eb[a_i^\top x_t] & = \Eb[a_i^\top x_t \Ib\{ \Ec \}] + \Eb[a_i^\top x_t \Ib\{ \Ec^c \}]\\
        & \leq b_i - \kappa + 2 S D \delta  \leq b_i,
    \end{aligned}}
\end{equation*}
where the first inequality uses that $a_i^\top x_t \leq b_i - \kappa$ under $\Ec$, that $a_i^\top x_t \leq S D$ always and that $\Pb(\Ec^c) \leq 2 \delta$, and the last inequality uses the fact that $2SD \delta \leq \kappa$.
Then, to characterize the impact that tightening the constraint has on the regret, we bound the distance between the optimal action $x^\star$ and the ``tightened'' feasible set $\tilde{\Yc} = \{ x \in \Xc : Ax \leq b - \kappa \mathbf{1} \}$.
Specifically, we show that this distance is smaller than $\kappa D / \bmin$, and therefore the algorithm only incurs a cost of $\Oc(1/T)$ per round, or $\Oc(1)$ in total, for tightening the constraint.

\begin{figure*}[t]
    \centering
    \includegraphics[width=\textwidth]{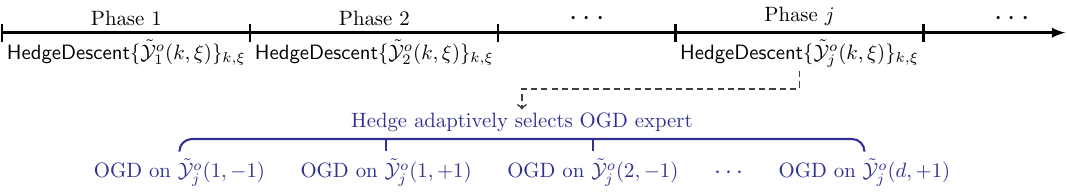}
    \vspace{0.01in}
    \caption{Implementation of HedgeDescent in OSOCO. OSOCO calls HedgeDescent in each phase and passes the family of optimistic sets $\{\Yctil_j^o(k,\xi)\}_{k,\xi}$ to it. HedgeDescent uses an online gradient descent (OGD) ``expert'' for each of these sets and then uses Hedge to adaptively select the OGD expert in each round.}
    \label{fig:hd_diag}
\end{figure*}

\section{An Efficient Version}
\label{sec:comp_eff}

One drawback of the original OSOCO algorithm is that it requires the use of an online nonconvex optimization algorithm, which often does not necessarily allow for computationally efficient implementations.
In this section, we modify the OSOCO algorithm to improve its computational efficiency, while only slightly weakening the regret guarantees to $\Octil(d^{3/2} \sqrt{T})$.
The key idea is that we modify the optimistic action set so that it can be represented as the union of a finite number of convex sets.
We then give an efficient online optimization algorithm that we call HedgeDescent (to be used for $\Ac$) that is designed for this type of structured nonconvex action set.
The implementation of HedgeDescent in OSOCO is depicted in Figure \ref{fig:hd_diag}.

\paragraph{Relaxation of Optimistic Action Set}
In order for the optimistic action set to be amenable to efficient online optimization, we use the idea from \cite{dani2008stochastic} of relaxing the least-squares confidence set.
In particular, from the equivalence of norms, the optimistic action set can be relaxed to,
$\Yctil_j^o = \big\{ x \in \Xc : \hat{A}_j x - \sqrt{d} \bar{\beta}_j \| \bar{V}_j^{-1/2} x \|_{\infty} \mathbf{1}  \leq b - \kappa \mathbf{1} \big\}$.
Although this set is still nonconvex, it can now be represented as the union of a finite number of convex sets.
To explicitly define these convex sets, we use the idea from \cite{cassel2022rate} of parameterizing the infinity norm by the active element in the vector and the sign of this element, i.e. $\| z \|_\infty = \max_{k \in [d], \xi \in \{-1,1\}} \xi z_k$  for $z \in \Rb^d$.
As such we can parameterize the optimistic set as,
\begin{equation*}
    \Yctil_j^o(k,\xi) := \big\{ x \in \Xc : \hat{A}_j x - \sqrt{d} \bar{\beta}_j \xi (\bar{V}_j^{-1/2})_k x \mathbf{1}  \leq b  - \kappa \mathbf{1} \big\},
\end{equation*}
where $(\bar{V}_j^{-1/2})_k$ denotes the $k$th row of $\bar{V}_j^{-1/2}$.
Evidently, $\Yctil_j^o(k,\xi)$ is convex for any fixed $k, \xi \in [d] \times \{ -1, 1 \}$ and furthermore the union over $k$ and $\xi$ is the original set, i.e. $\bigcup_{k, \xi \in [d] \times \{ -1, 1 \}} \Yctil_j^o(k,\xi) = \Yctil_j^o$.
Accordingly, we modify OSOCO such that it sends the family of sets $\{ \Yctil_j^o(k,\xi) \}_{k, \xi}$ to $\Ac$ in line \ref{lne:start_hd}.

\paragraph{HedgeDescent Algorithm}
\label{sec:hd_alg}

In order to appropriately exploit the structure of the relaxed optimistic action set, we propose the online algorithm HedgeDescent (Algorithm \ref{alg:hedg_desc}) to be used for $\Ac$ in OSOCO.
Since the optimistic action set is now the union of a finite number of convex sets, HedgeDescent uses a finite online algorithm (i.e. Hedge \citep{freund1997decision}) to choose which convex set to play in and a continuous online algorithm (i.e. online gradient descent \citep{zinkevich2003online}) to ensure low regret within each convex set.\footnote{This approach is inspired by \cite{cassel2022rate}, where a similar idea is used to handle nonconvexity in the cost rather than nonconvexity in the constraint.}
As such, HedgeDescent only requires $M$ convex projections (where $M$ is the number of convex sets).
When OSOCO is modified such that the optimistic set is $\{ \Yctil_j^o(k,\xi) \}_{k, \xi}$, then $M = 2d$ and thus the combined algorithm requires $2d$ convex projections per round.

\begin{algorithm}[h]
    \caption{HedgeDescent}
    \label{alg:hedg_desc}
\begin{algorithmic}[1]
    \INPUT $\{\Xc_m\}_{m \in [M]}$.
    \STATE Initialize: $x_t(m) = \mathbf{0}$, $p_t(m) = 1/M \ \forall m \in [M]$.
    \FOR{$t = 1$ \textbf{ to } $T$}
        \STATE Sample Hedge: $m_t \sim p_t$.\alglabel{lne:phase_start}
        \STATE Play $x_t(m_t)$ and observe $f_t$.
        \STATE Update experts:\\ $x_{t+1} (m) = \Pi_{\Xc_m}\left(x_t (m) - \eta_t \nabla f_t(x_t (m)) \right)$ for all $m \in [M]$.\alglabel{lne:upd_exp}
        \STATE Update Hedge:\\ $p_{t+1} (m) \propto p_{t} (m) \exp(-\zeta_t f_t(x_t (m)))$.\alglabel{lne:upd_dist}
    \ENDFOR
\end{algorithmic}
\end{algorithm}

\paragraph{Analysis}

The following proposition (proven in Appendix \ref{sec:pf_effic}) shows that with the aforementioned modifications, OSOCO enjoys $\Octil(d^{3/2} \sqrt{T})$ regret and no violation (in either the high-probability or expected sense).
Although this is slightly worse than the $\Octil(d \sqrt{T})$ regret attained by the original OSOCO algorithm (Theorem \ref{thm:main_reg} and Proposition \ref{prop:expec_viol}), the modified algorithm only requires $2 d$ convex projections (which is polynomial-time) in each round instead of updating a grid over the action set or performing non-convex optimization (neither of which is polynomial-time) as required by the original algorithm.

\begin{proposition}
    \label{prop:relax_reg}
    If OSOCO is modified such that $\Yc_j^o \leftarrow \{ \Yctil_j^o(k,\xi) \}_{k, \xi}$ and HedgeDescent is used for $\Ac$, then the following hold:
    \begin{itemize}
        \item \emph{(High-probability)} Choosing the algorithm parameters as in Theorem \ref{thm:main_reg} ensures, with probability at least $1 - 2 \delta$, that $A x_t \leq b$ and $R_T \leq \Octil(d^{3/2} \sqrt{T})$.
        \item \emph{(Expectation)} Choosing the algorithm parameters as in Proposition \ref{prop:expec_viol} ensures that $\Eb[A x_t] \leq b$ and $\Eb[R_T] \leq \Octil(d^{3/2} \sqrt{T})$.
    \end{itemize}
\end{proposition}

\section{Extension to OCO with Stochastic Linear Constraints}
\label{sec:tvar}

In this section, we consider the problem of OCO with stochastic linear constraints and bandit constraint feedback, and show that OSOCO guarantees $\Octil(\sqrt{T})$ expected regret and no expected constraint violation.

\subsection{Problem Setup}
\label{sec:stoch_prob}

In each round $t \in [T]$, the player chooses an action $x_t$ from $\Xc$ and, simultaneously, the adversary choose the convex cost function $f_t : \Xc \rightarrow \Rb$ according to the actions chosen by the player in previous rounds.
The player then receives full cost feedback $f_t$ and bandit constraint feedback $g_t(x_t)$.
The constraint function is of the form $g_t(x) = A_t x - b_t$, where $(A_t, b_t)_{t = 1}^T$ is an i.i.d. sequence in $\Rb^{n \times d} \times \Rb^n$.
The player aims to ensure \emph{no expected violation},
\begin{equation*}
    \Eb[g_t(x_t)] \leq 0 \quad \forall t \in [T],
\end{equation*}
and to minimize the regret,
\begin{equation*}
    R_T^{\mathrm{stoch}} := \sum_{t \in [T]} f_t(x_t) - \min_{x \in \Xc : \bar{g}(x) \leq 0} \sum_{t \in [T]} f_t(x),
\end{equation*}
where $\bar{g}(x) := \Eb_{A_t,b_t} [g_t(x)] = \Eb[A_t] x - \Eb[b_t]$.

Note that we take $(A_t,b_t)$ to be stochastic in this case, whereas we consider deterministic $(A,b)$ in the setting with static constraints (i.e. Section \ref{sec:prob_set}).

We use the following the assumptions on the constraint and will assume that the cost functions and action set satisfy Assumptions \ref{ass:cost_funcs} and \ref{ass:set_bound}.

\begin{assumption}[Assumptions on Stochastic Constraint]
    \label{ass:stoch}
    With $g_{t,i}$ denoting the $i$th component of $g_t$, assume that,
    \begin{enumerate}[(i)]
        \item $\| \nabla g_{t,i} (\mathbf{0}) \| \leq G_g$ for all $i \in [n], t \in [T]$, \label{item:abound}
        \item $\| g_t(\mathbf{0}) \| \leq F$ for all $t \in [T]$, \label{item:bbound}
        \item $\bar{g}(\mathbf{0}) < 0$,\label{item:slaters}
        % \item $A_t$ and $b_t$ are bounded for all $t \in [T]$,\label{item:bounded}
        \item $\bar{g}(\mathbf{0})$ is known.\label{item:known}
    \end{enumerate}
\end{assumption}

Condition (\ref{item:abound}) and Condition (\ref{item:bbound}) are standard in OCO with stochastic constraints, e.g. \citep{yu2017online,yu2023online}. 
Condition (\ref{item:slaters}) requires both the existence of a strictly feasible point in expectation (i.e. Slater's condition) which is commonly used in OCO with stochastic constraints (e.g. \cite{yu2017online,castiglioni2022unifying}), as well as knowledge of such a strictly feasible point in expectation.
Although the knowledge of a strictly feasible point and its expected function value (i.e. Condition (\ref{item:known})) is not typically assumed in OCO with stochastic constraints, such an assumption is often made in safe learning problems that ensure no violation, e.g. \citep{moradipari2021safe,amani2021safe}.

\subsection{Regret and Violation Guarantees}

The regret and violation guarantees of OSOCO in this setting in the following corollary to Proposition \ref{prop:expec_viol}.
In particular, Corollary \ref{cor:tvar} (proven in Appendix \ref{sec:tvar_apx}) guarantees no expected violation and $\Octil((d + C_{\Ac} \sqrt{d}) \sqrt{T})$ expected regret in general, and $\Octil(d^{3/2} \sqrt{T})$ expected regret for the more efficient version in Section \ref{sec:comp_eff}.

\begin{corollary}
    \label{cor:tvar}
    Suppose that the cost functions and action set satisfy Assumptions \ref{ass:cost_funcs} and \ref{ass:set_bound}.
    Also, assume that the constraint function satisfies Assumption \ref{ass:stoch}, and let $\rho^2 = (G_g D + 4 F)^2/4$.
    Then playing OSOCO (Algorithm \ref{alg:main_alg}) with the algorithm parameters chosen as in Proposition \ref{prop:expec_viol} ensures that $\Eb[g_t(x_t)] \leq 0$ for all $t \in [T]$ and $\Eb[R_T^{\mathrm{stoch}}] = \Octil((d + C_{\Ac} \sqrt{d}) \sqrt{T})$.
    Furthermore, with modifications in Proposition \ref{prop:relax_reg}, it holds that $\Eb[R_T^{\mathrm{stoch}}] = \Octil(d^{3/2} \sqrt{T})$.
\end{corollary}

\begin{figure*}[t]
    \centering
    % \hfill
    \begin{subfigure}[t]{0.28\textwidth}
        \centering    
        % \vspace{-0.1in}
        \includegraphics[width=\columnwidth]{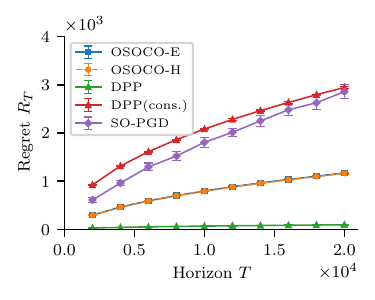}
        % \vspace{-0.3in}
        \caption{}
        \label{fig:expers:a}
    \end{subfigure}
    \hspace{0.035\textwidth}
    \begin{subfigure}[t]{0.28\textwidth}
        \centering    
        % \vspace{-0.1in}
        \includegraphics[width=\columnwidth]{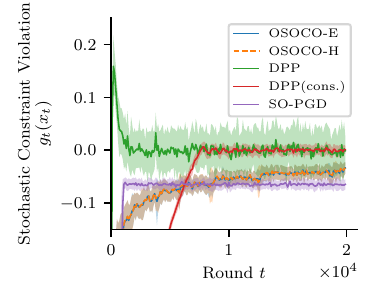}
        % \vspace{-0.3in}
        \caption{}
        \label{fig:expers:b}
    \end{subfigure}
    \hspace{0.035\textwidth}
    % ~
    \begin{subfigure}[t]{0.28\textwidth}
        \centering    
        % \vspace{-0.1in}
        \includegraphics[width=\columnwidth]{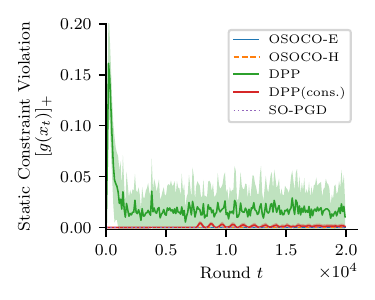}
        % \vspace{-0.3in}
        \caption{}
        \label{fig:expers:c}
    \end{subfigure}
    % \
    % \vspace{-0.1in}
   \caption{Simulation results for our algorithm OSOCO configured for high-probability (OSOCO-H) and expectation (OSOCO-E) guarantees, alongside DPP from \cite{yu2017online} and SO-PGD from \cite{chaudhary2022safe}. The points indicate the average over 30 trials and the error bars and shading are $\pm 1$ standard deviation.}
   \label{fig:expers}
%    \vspace{-0.1in}
\end{figure*}

\section{Numerical Experiments}

\label{sec:num_exp}

In order to further validate the effectiveness of OSOCO, we perform numerical experiments in a toy setting.
Due to space limitations, we defer the experiment details to Appendix \ref{sec:det_num} and just give a high level explanation here.
We study OSOCO configured for both high-probability (OSOCO-H) and expectation (OSOCO-E) guarantees, as well as the algorithm from \cite{yu2017online} with two different choice of parameters (DPP, DPP(cons)) and \cite{chaudhary2022safe} (SO-PGD).
The results are shown in Figure \ref{fig:expers}, where average (over $30$ trials) is indicated by a point or line and the standard deviation is indicated by an errorbar or shading.
DPP has the lowest regret in all experiments, although also has sizeable constraint violation.
OSOCO-E and OSOCO-H have the next lowest regret, and do not incur stochastic violation on average, or static violation at all.

\section{Future Work}
\label{sec:conc}

Some promising directions for future work are (1) considering more general constraint models, such as gaussian process models, and (2) applying our approach to related safe learning problems, such as online control or MDPs with unknown constraints.

\bibliographystyle{plainnat}
\bibliography{references}

\newpage

\tableofcontents

\newpage

\appendix

\section{More Related Work}
\label{sec:more_rel_work}

In this section, we discuss the relationship with literature on OCO with long-term constraints, projection-free OCO literature and online control.

\paragraph{OCO with Long-term Constraints}

The literature on OCO with long-term constraints considers the problem of OCO with fixed constraints where it is not necessary to maintain the feasibility of actions, but instead only to ensure sublinear cumulative violation, e.g. \citep{mahdavi2012trading,yu2020low,guo2022online}.
Some methods in this literature only use the constraint function value and gradient at the played action, which is a form of limited feedback on the constraint, e.g. \cite{mahdavi2012trading,yu2017online,yuan2018online}.
Nonetheless, this is a richer form of feedback than we consider (noisy function value) and therefore these methods are not applicable to our setting.
Furthermore, these methods do not ensure feasibility, and therefore are not applicable for the guarantees that we are interested in.

\paragraph{Projection-free OCO}

The problem of OCO under unknown constraints has similar goals to that of projection-free OCO, as both settings aim to develop effective methods for OCO when the action set cannot be accessed through a projection oracle.
A large number of projection-free OCO methods use the cheaper linear optimization oracle instead of the projection oracle and often use variants of the Franke-Wolfe method, e.g. \cite{hazan2012projection,garber2016linearly,hazan2020faster,kretzu2021revisiting}.
Some other approaches use the membership oracle \cite{mhammedi2022efficient} or the constraint function and gradient at any point \cite{levy2019projection}.
The literature on projection-free OCO is distinct from OCO under unknown constraints in that they consider different types of oracles that use the constraints, e.g. membership oracle, constraint function value and gradient at any point, or linear optimization oracle.

\paragraph{Online Control}

There is an interesting connection between OCO with unknown linear constraints and online control with unknown linear dynamics (e.g. \cite{cohen2018online,simchowitz2020improper,cassel2022rate}).
In particular, both settings have unknown linear functions with noisy feedback and adversarial convex costs.
However, in online control, the unknown linear function impacts the incurred cost, while in OCO with unknown linear constraints, the unknown linear function forms the constraint.
The latter setting poses a fundamentally different challenge that we aim to address in this work.

\section{Proof of Theorem \ref{thm:main_reg} (High-probability Guarantees)}
\label{sec:main_proof}

In this section, we give the proof of Theorem \ref{thm:main_reg}.
We first give the notation and assumptions in Section \ref{sec:notations} and state the high probability events in Section \ref{sec:high_prob}.
Then, we give the bound on \cosa{} (Cost of Safety) in Section \ref{sec:reg_safety} and the bound on \opr{} (Optimistic Regret) in Section \ref{sec:opt_reg}.
We then complete the proof of Theorem \ref{thm:main_reg} (Regret of OSOCO) in Section \ref{sec:compl_proof}.
Lastly, side lemmas are given in Section \ref{sec:side_ls}.

\subsection{Notation and Assumptions}

\label{sec:notations}

In addition to the notation given in the body of the paper, we use $N$ to refer to the total number of phases and $t_j$ to refer to the first round within phase $j$ where $t_{N+1} := T$.
Additionally, we use $j_t$ to refer to the phase to which round $t$ belongs, i.e. $j_t = \max\{j \in [N] : t_j \leq t\}$.
Lastly, we use the notation $\Hc_t := \sigma(\epsilon_1,..,\epsilon_t, \xtil_1, ..., \xtil_t)$, i.e. the randomness in the first $t$ rounds (due to both the environment and algorithm), and note that $\Eb[\cdot|\Hc_{t-1}] = \Eb_t[\cdot]$.

In addition to the assumptions made in the body, we will also assume that the assumptions made in Theorem \ref{thm:main_reg} hold, i.e. $\kappa = 0$, $\lambda = \max(1, D^2)$, $\beta_t := \rho \sqrt{ d \log \left(\frac{1 + (t - 1) L^2/\lambda}{\delta/n} \right)} + \sqrt{\lambda} S$, that $\Ac$ enjoys a regret constant $C_\Ac$ (in the sense of Definition \ref{def:reg}) and that $T \geq 3$.

\subsection{High probability events}

\label{sec:high_prob}

In this section, we show that with an appropriate choice of problem parameters, the events
\begin{equation*}
    \Econf := \left\{ (A - \hat{A}_j) x \in \bar{\beta}_j \|  x \|_{\bar{V}_j^{-1}} \Bb_\infty, \forall x \in \Rb^d, \forall j \in [N] \right\},
\end{equation*}
and
\begin{equation*}
    \Eazuma := \left\{ \sum_{t=1}^{T} (f_t(\xtil_t) - \Eb[f_t(\xtil_t) | \Hc_{t-1}]) \leq 2 DG \sqrt{2 T \log(1/\delta)} \right\},
\end{equation*}
occur with high probability.
First, the following theorem from \cite{abbasi2011improved} tells us that $\Pb(\Econf) \geq 1 - \delta$.

\begin{theorem}[Theorem 2 in \cite{abbasi2011improved}]
    \label{thm:conf_set}
    Fix $\delta \in (0,1)$ and let $\check{A}_t = S_t V_t^{-1}$ and,
    \begin{equation*}
        \beta_t := \rho \sqrt{ d \log \left(\frac{1 + (t - 1) D^2/\lambda}{\delta/n} \right)}.
    \end{equation*}
    Then, with probability at least $1 - \delta$, it holds for all $t \geq 1$ and $x \in \Rb^d$ that $(\check{A}_t - A)x \in \beta_t \| x \|_{V_t^{-1}} \Bb_\infty$.
\end{theorem}

Then, an application of Azuma's inequality yields a bound on the probability of $\Eazuma$ as shown in the following lemma.

\begin{lemma}[Probability of $\Eazuma$]
    \label{lem:azuma}
    The event $\Eazuma$ occurs with probability at least $1 - \delta$.
\end{lemma}
\begin{proof}
    Let
    \begin{equation*}
        Z_t := \sum_{s=1}^{t} \left( f_s(\xtil_s) - \Eb[f_s(\xtil_s) | \Hc_{s-1}] \right)
    \end{equation*}
    We prove this lemma by proving that $Z = (Z_t)_t$ is a martingale with respect to $\Hc = (\Hc_t)_t$ and has bounded differences.
    Then, an application of Azuma's inequality gives the statement of the lemma.

    First, note that $f_t$ is fully determined by the randomness in rounds $[t-1]$ and $\xtil_t$ is fully determined by the randomness in rounds $[t]$, so $f_t(\xtil_t)$ is $\Hc_t$-measurable.
    Therefore, $Z_t$ is $\Hc_t$-measurable and we say that $Z$ is adapted to $\Hc$.
    Furthermore, it holds that
    \begin{align*}
        \Eb[Z_t | \Hc_{t-1}] & = \Eb \left[ \sum_{s=1}^{t} \left( f_s(\xtil_s) - \Eb[f_s(\xtil_s) | \Hc_{s-1}] \right) \ \bigg|\  \Hc_{t-1} \right]\\
        & = \sum_{s=1}^{t-1} \left( f_s(\xtil_s) - \Eb[f_s(\xtil_s) | \Hc_{s-1}] \right) +  \Eb \left[ f_t(\xtil_t) | \Hc_{t-1}\right] - \Eb[f_t(\xtil_t) | \Hc_{t-1}]\\
        & = Z_{t - 1}
    \end{align*}
    and therefore $Z$ is a martingale with respect to $\Hc$.
    Then, by fixing $y \in \Xc$, it follows that
    \begin{equation}
        \label{eqn:cond_bound}
        \begin{split}
            |Z_t - Z_{t-1}| & = |f_t(\xtil_t) - \Eb[f_t(\xtil_t) | \Hc_{t-1}]|\\
            & = |f_t(\xtil_t) - f_t(y) + f_t(y) - \Eb[f_t(\xtil_t) | \Hc_{t-1}]|\\
            & = |f_t(\xtil_t) - f_t(y) + \Eb[f_t(y) - f_t(\xtil_t) | \Hc_{t-1}]| \\
            & \leq |f_t(\xtil_t) - f_t(y)| + \Eb[|f_t(y) - f_t(\xtil_t)| | \Hc_{t-1}] \\
            & \leq 2 G \| \xtil_t - y \| \\
            & \leq 2 GD,
        \end{split}
    \end{equation}
    where we the first inequality uses Jensen's inequality, the second inequality uses the assumption of bounded gradients as specified in Assumption \ref{ass:cost_funcs}, and the third inequality uses the assumption of bounded action set from Assumption \ref{ass:set_bound}.
    Therefore, we can apply Azuma's inequality to get that
    \begin{equation*}
        \Pb(Z_T > \nu) \leq \exp\left( \frac{-2\nu^2}{4 T G^2 D^2}  \right),
    \end{equation*}
    for any $\nu > 0$.
    Setting the right-hand side to $\delta$ gives that $\Pb(\Eazuma) \geq 1 - \delta$.
\end{proof}

\subsection{Term I (Cost of Safety)}

\label{sec:reg_safety}

Before proving the bound on Term I, we first give a key lemma lower bounding $\gamma_t$.

\begin{lemma}[Bound on $\gamma_t$]
    \label{lem:gam_bound}
    It holds for all $t \in [T]$ that
    \begin{equation*}
        \gamma_t \geq 1 - \frac{2}{\bmin} \bar{\beta}_{j_t} \| x_t \|_{\bar{V}_{j_t}^{-1}}.
    \end{equation*}
\end{lemma}
\begin{proof}
    Throughout the proof, we we take $j = j_t$ to simplify notation.
    Then to establish the proof, we first show that $\alpha \xtil_t \in \Yc_j^p$ where
    \begin{equation*}
        \alpha := \frac{\bmin}{\bmin + 2 \bar{\beta}_j \| \xtil_t \|_{\bar{V}_j^{-1}}}
    \end{equation*}
    and then complete the proof by showing that this implies the statement of the lemma.
    Each of these steps are given in the following sections.

    \paragraph{Proof that $\alpha \xtil_t \in \Yc_j^p$:}
    It holds that $\alpha \xtil_t \in \Yc_j^p$ if both (a) $\alpha \xtil_t \in \Xc$ and (b) $\hat{a}_{j,i}^\top (\alpha \xtil_t) + \bar{\beta}_j \|\alpha \xtil_t \|_{\bar{V}_j^{-1}} \leq b_i$ for all $i \in [n]$.
    Point (a) holds because $\alpha \in [0,1]$, $\xtil_t \in \Xc$, $\mathbf{0} \in \Xc$ and $\Xc$ is convex, and therefore
    \begin{equation*}
        \alpha \xtil_t = \alpha \xtil_t + (1 - \alpha) \mathbf{0} \in \Xc.
    \end{equation*}
    Point (b) holds given that for all $i \in [n]$,
    \begin{align*}
        \hat{a}_{j,i}^\top (\alpha \xtil_t) + \bar{\beta}_j \| \alpha \xtil_t \|_{\bar{V}_j^{-1}} & = \alpha (\hat{a}_{j,i}^\top \xtil_t + \bar{\beta}_j \| \xtil_t \|_{\bar{V}_j^{-1}})\\
        & = \alpha (\hat{a}_{j,i}^\top \xtil_t - \bar{\beta}_j \| \xtil_t \|_{\bar{V}_j^{-1}} + 2 \bar{\beta}_j \| \xtil_t \|_{\bar{V}_j^{-1}})\\
        & \leq \alpha (b_i  + 2 \bar{\beta}_j \| \xtil_t \|_{\bar{V}_j^{-1}}) \tag{a} \label{eqn:gam_a}\\
        & = \frac{\bmin }{\bmin  + 2 \bar{\beta}_j \| \xtil_t \|_{\bar{V}_j^{-1}}} (b_i  + 2 \bar{\beta}_j \| \xtil_t \|_{\bar{V}_j^{-1}}) \tag{b} \label{eqn:gam_b}\\
        & \leq \frac{b_i }{b_i  + 2 \bar{\beta}_j \| \xtil_t \|_{\bar{V}_j^{-1}}} (b_i  + 2 \bar{\beta}_j \| \xtil_t \|_{\bar{V}_j^{-1}}) \tag{c} \label{eqn:gam_c}\\
        & = b_i 
    \end{align*}
    where \eqref{eqn:gam_a} uses the fact that $\xtil_t \in \Yc_j^o$, \eqref{eqn:gam_b} uses the definition of $\alpha$, and \eqref{eqn:gam_c} uses the fact that $\frac{\bmin }{\bmin  + 2 \bar{\beta}_j \| \xtil_t \|_{\bar{V}_j^{-1}}}$ is increasing in $\bmin$ and $\bmin  \leq b_i $.

    \paragraph{Completing the proof:}
    From the definition of $\gamma_t$, it holds that
    \begin{equation}
        \label{eqn:above}
        \gamma_t = \max\{\alpha \in [0,1] : \alpha \xtil_t \in \Yc_{j}^p\} \geq \frac{\bmin }{\bmin  + 2 \bar{\beta}_{j} \| \xtil_t \|_{\bar{V}_{j}^{-1}}}.
    \end{equation}
    Then, since $x_t = \gamma_t \xtil_t$ and $\gamma_t \geq 0$, it also holds that
    \begin{equation*}
        \gamma_t \|  \xtil_t \|_{\bar{V}_{j}^{-1}} = \|  \gamma_t \xtil_t \|_{\bar{V}_{j}^{-1}} = \|  x_t \|_{\bar{V}_{j}^{-1}}.
    \end{equation*}
    Therefore, we can rearrange \eqref{eqn:above} to get that
    \begin{equation*}
        \gamma_t \geq 1 - 2 \frac{1}{\bmin } \bar{\beta}_{j} \|  x_t \|_{\bar{V}_{j}^{-1}},
    \end{equation*}
    completing the proof.
\end{proof}

With this, we then give the bound on Term I.

\begin{lemma}[Cost of Safety]
    \label{lem:reg_safety2}
    It holds that
    \begin{equation}
        \label{eqn:reg_safety2}
        \sum_{t=1}^{T} f_t(x_t) - \sum_{t=1}^{T} f_t(\xtil_t) \leq 4 D G \frac{1}{\bmin } \beta_T \sqrt{3 d T \log \left( T \right)}.
    \end{equation}
\end{lemma}
\begin{proof}
    We can study the lefthand side of \eqref{eqn:reg_safety2} to get that
    \begin{align*}
        \sum_{t=1}^{T} \left( f_t(x_t) - f_t(\xtil_t) \right) & \leq G \sum_{t=1}^{T} \| \xtil_t -  x_t \| \\
        & = G \sum_{t=1}^{T} \| \xtil_t -  \gamma_t \xtil_t \| \tag{a} \label{eqn:multi2_a}\\
        & \leq D G \sum_{t=1}^{T} (1 - \gamma_t) \\
        & \leq 2 D G \frac{1}{\bmin }  \sum_{t=1}^{T} \bar{\beta}_{j_t} \|  x_t \|_{\bar{V}_{j_t}^{-1}} \tag{b} \label{eqn:multi2_b}\\
        & \leq 2 D G \frac{1}{\bmin } \beta_{T}  \sum_{t=1}^{T} \|  x_t \|_{\bar{V}_{j_t}^{-1}} \tag{c} \label{eqn:multi2_c}\\
        & \leq 2 D G \frac{1}{\bmin } \beta_{T}  \sum_{t=1}^{T} \frac{\det(V_t)}{\det(\bar{V}_{j_t})} \|  x_t \|_{V_t^{-1}} \tag{d} \label{eqn:multi2_d} \\
        & \leq 4 D G \frac{1}{\bmin } \beta_T \sum_{t=1}^{T}  \| x_t \|_{V_{t}^{-1}} \tag{e} \label{eqn:multi2_e}\\
        & \leq 4 D G \frac{1}{\bmin } \beta_T \sqrt{T \sum_{t=1}^{T}  \| x_t \|_{V_{t}^{-1}}^2} \tag{f} \label{eqn:multi2_f}\\
        & \leq 4 D G \frac{1}{\bmin } \beta_T \sqrt{2 d T \log\left(1 + \frac{T}{\lambda d} \right)} \tag{g} \label{eqn:multi2_g}\\
        & \leq 4 D G \frac{1}{\bmin } \beta_T \sqrt{3 d T \log \left(T \right)} \tag{h} \label{eqn:multi2_h},
    \end{align*}  
    where each step is justified in the following:
    \begin{itemize}
        \item[\eqref{eqn:multi2_a}] $x_t = \gamma_t \xtil_t$ as specified in line \ref{lne:play_act} in the algorithm.
        \item[\eqref{eqn:multi2_b}] Lemma \ref{lem:gam_bound}.
        \item[\eqref{eqn:multi2_c}] $\beta_t$ is increasing in $t$.
        \item[\eqref{eqn:multi2_d}] For pd matrices $B \preceq C$ it holds that $\|z\|_{B^{-1}} \leq \frac{\det(C)}{\det(B)}\| z \|_{C^{-1}}$ for any $z$ (see e.g. Lemma 12 in \cite{abbasi2011improved} or Lemma 27 in \cite{cohen2019learning}).
        \item[\eqref{eqn:multi2_e}] Follows from phase termination criteria (line \ref{lne:phase_end}) which ensures that $\det(V_t) \leq 2 \det(\bar{V}_{j_t})$ for all $t$.
        \item[\eqref{eqn:multi2_f}] Cauchy-Schwarz.
        \item[\eqref{eqn:multi2_g}] Elliptic potential lemma (Lemma \ref{lem:elliptic}).
        \item[\eqref{eqn:multi2_h}] $2 d\log\left(1 + \frac{T}{\lambda d} \right) \leq  2 d \log\left(1 + T \right) \leq 3 d \log\left(T \right)$ when $T \geq 3$ and $\lambda = \max(1,D^2)$.
    \end{itemize}
\end{proof}

\subsection{Term II (Optimistic Regret)}

\label{sec:opt_reg}

We prove a bound on Term II in the following.

\begin{lemma}[Optimistic Regret]
    \label{lem:opt_reg2}
    Conditioned on $\Econf$ and $\Eazuma$, it holds that
    \begin{align*}
        \sum_{t=1}^{T} f_t(\xtil_t) - \sum_{t=1}^{T} f_t(x^\star) \leq  C_{\Ac} \sqrt{4 d T \log(T)} + D G \sqrt{2 T \log(1/\delta)}  
    \end{align*}
\end{lemma}
\begin{proof}
    Let $x^\star_j$ be an optimal action within the optimistic set in phase $j$, i.e. $x^\star_j \in \argmin_{x \in \Yc_j^o} \sum_{t=t_j}^{t_{j+1}-1} f_t(x)$.
    Then, consider the decomposition,
    \begin{align*}
        \sum_{t=1}^{T} f_t(\xtil_t) - \sum_{t=1}^{T} f_t(x^\star) = & \underbrace{\sum_{t=1}^{T} f_t(\xtil_t) - \sum_{t=1}^{T} \Eb[f_t(\xtil_t) | \Hc_{t-1}]}_{\tone}\\
        & + \underbrace{\sum_{t=1}^{T} \Eb[f_t(\xtil_t) | \Hc_{t-1}] - \sum_{j=1}^{N} \sum_{t=t_j}^{t_{j+1}-1} f_t(x^\star_j)}_{\ttwo}\\
        & + \underbrace{\sum_{j=1}^{N} \sum_{t=t_j}^{t_{j+1}-1} f_t(x^\star_j) - \sum_{t=1}^{T} f_t(x^\star).}_{\tthree}
    \end{align*}
    We bound each of the terms in the following sections.

    \paragraph{Term I:}
    Conditioning on $\Eazuma$, it holds by definition that
    \begin{equation}
        \tone = \sum_{t=1}^{T} f_t(\xtil_t) - \sum_{t=1}^{T} \Eb[f_t(\xtil_t)| \Hc_{t-1}] \leq 2 D G \sqrt{2 T \log(1/\delta)}.
    \end{equation}

    \paragraph{Term II:}
    We bound Term II as follows,
    \begin{align*}
        \ttwo & = \sum_{j=1}^{N} \sum_{t=t_j}^{t_{j+1}-1} \Eb[f_t(\xtil_t)| \Hc_{t-1}] - \sum_{j=1}^{N} \sum_{t=t_j}^{t_{j+1}-1} f_t(x^\star_j)\\
        & = \sum_{j=1}^{N} \sum_{t=t_j}^{t_{j+1}-1} \Eb_{\xtil_t \sim p_t}[f_t(\xtil_t)] - \sum_{j=1}^{N} \sum_{t=t_j}^{t_{j+1}-1} f_t(x^\star_j) \tag{a} \label{eqn:phase_a} \\
        & \leq \sum_{j=1}^{N} C_{\Ac} \sqrt{t_{j+1} - t_{j}} \tag{b} \label{eqn:phase_b} \\
        & \leq C_{\Ac} \sum_{j=1}^{N} \sqrt{t_{j+1} - t_{j}} \tag{c} \label{eqn:phase_c}\\
        & \leq C_{\Ac} \sqrt{N \sum_{j=1}^{N} (t_{j+1} - t_{j}) } \tag{d} \label{eqn:phase_d}\\
        & = C_{\Ac} \sqrt{N T} \\
        & \leq C_{\Ac} \sqrt{4 d \log(T) T} \tag{e} \label{eqn:phase_e}
    \end{align*}
    where each step is justified in the following:
    \begin{itemize}
        \item[\eqref{eqn:phase_a}]
        Note that the the distribution from which $\xtil_t$ is sampled (denoted by $p_t$) is fully determined by the randomness in rounds $[t-1]$ and therefore $p_t$ is the conditional distribution of $\xtil_t$ given $\Hc_{t-1}$.
        Additionally, the adversary chooses $f_t$ simultaneous to the player choosing $x_t$ and therefore $f_t$ is determined by the randomness in rounds $[t - 1]$.
        Therefore, the conditional distribution of $f_t(\xtil_t)$ given $\Hc_{t-1}$ is $p_t$, and thus $\Eb[f_t(\xtil_t)| \Hc_{t-1}] = \Eb_{\xtil_t \sim p_t}[f_t(\xtil_t)]$.
        \item[\eqref{eqn:phase_b}] The regret bound of $\Ac$ as specified in Definition \ref{def:reg}. Note that the duration of each phase ($t_{j+1} - t_j$) depends on the actions chosen in that phase. Nonetheless, $\Ac$ is specified to handle a setting that can be terminated at any time and therefore the regret guarantees hold even when the duration of the phase depends on the actions taken.
        \item[\eqref{eqn:phase_c}] The fact that $C_\Ac$ is polylogarithmic in $t_{j+1} - t_{j}$ and therefore (with abuse of notation) pull it out of the summation to mean that it is polylogarithmic in $T$ as $t_{j+1} - t_{j} \leq T$.
        \item[\eqref{eqn:phase_d}] Cauchy-Schwarz.
        \item[\eqref{eqn:phase_e}] $N \leq 4 d \log(T)$ from Lemma \ref{lem:num_phases}.
    \end{itemize}

    \paragraph{Term III:}
    Conditioned on $\Econf$, it holds that
    \begin{align*}
        \tthree & = \sum_{j=1}^{N} \sum_{t=t_j}^{t_{j+1}-1} f_t(x^\star_j) - \sum_{t=1}^{T} f_t(x^\star)\\
        & = \sum_{j=1}^{N} \min_{x \in \Yc_j^o} \sum_{t=t_j}^{t_{j+1}-1} f_t(x) - \sum_{t=1}^{T} f_t(x^\star)\\
        & \leq \sum_{j=1}^{N} \min_{x \in \Yc}\sum_{t=t_j}^{t_{j+1}-1} f_t(x) - \sum_{t=1}^{T} f_t(x^\star)\\
        & \leq \min_{x \in \Yc} \sum_{t=1}^{T} f_t(x) - \sum_{t=1}^{T} f_t(x^\star) = 0,
    \end{align*}
    where the first inequality uses $\Yc \subseteq \Yc_j^o$ for all $j$ which follows from Theorem \ref{thm:conf_set}.

    \paragraph{Completing the proof:}
    Putting everything together, we get, conditioned on $\Econf$ and $\Eazuma$ that
    \begin{align*}
        \sum_{t=1}^{T} f_t(\xtil_t) - \sum_{t=1}^{T} f_t(x^\star) \leq  C_{\Ac} \sqrt{4 d T \log(T)} + D G \sqrt{2 T \log(1/\delta)},  
    \end{align*}
    completing the proof.
\end{proof}

\subsection{Completing the proof}

\label{sec:compl_proof}

Finally, we restate Theorem \ref{thm:main_reg} in the following.

\begin{theorem}[Restatement of Theorem \ref{thm:main_reg}]
    \label{thm:main_reg2}
    Fix some $\delta \in (0,1/2)$.
    Then, with probability at least $1 - 2 \delta$, the actions of OSOCO (Algorithm~\ref{alg:main_alg}) satisfy, 
    \begin{equation*}
        R_T \leq  4 \frac{D G}{\bmin} \beta_T \sqrt{3 d T \log \left( T \right)} + C_{\Ac} \sqrt{4 d T \log(T)} + 2 D G \sqrt{2 T \log(1/\delta)} 
    \end{equation*}
    and $A x_t \leq b$ for all $t \in [T]$.
\end{theorem}
\begin{proof}
    We can apply Lemmas \ref{lem:reg_safety2} and \ref{lem:opt_reg2} and condition on $\Econf$ and $\Eazuma$ to get that
    \begin{align*}
        R_T & = \sum_{t=1}^T f_t(x_t) - \sum_{t=1}^T f_t(x^\star)\\
        & = \underbrace{\sum_{t=1}^T f_t(x_t) - \sum_{t=1}^T f_t(\xtil_t)}_{\text{Term I}} + \underbrace{\sum_{t=1}^T f_t(\xtil_t) - \sum_{t=1}^T f_t(x^\star)}_{\text{Term II}}\\
        & \leq 4 D G \frac{1}{\bmin} \beta_T \sqrt{2 d T \log \left( T \right)} + 2 D G \sqrt{2 T \log(1/\delta)} + C_{\Ac} \sqrt{4 d \log(T) T}.
    \end{align*}
    Also, conditioned on $\Econf$, it holds that $x_t \in \Yc_{j_t}^p \subseteq \Yc := \{ x \in \Xc : A x \leq b \}$ for all $t \in [T]$ and therefore $A x_t \leq b$ for all $t \in [T]$.
    To get a bound of the probability of both events holding, we use Theorem \ref{thm:conf_set} and Lemma \ref{lem:azuma} with the union bound to get that
    \begin{equation*}
        \Pb(\Econf \cap \Eazuma) = 1 - \Pb(\Econf^c \cup \Eazuma^c) \geq 1 - (\Pb(\Econf^c) + \Pb(\Eazuma^c)) \geq 1 - 2 \delta,
    \end{equation*}
    completing the proof.
\end{proof}

\subsection{Side lemmas}

\label{sec:side_ls}

In this section, we give side lemmas that were used in the proof.
First, we have the well-known elliptic potential lemma.
The version that we give is from \cite{abbasi2011improved}.

\begin{lemma}[Elliptic Potential, Lemma 11 in \cite{abbasi2011improved}]
    \label{lem:elliptic}
    Consider a sequence $(w_t)_{t \in \Nb}$ where $w_t \in \Rb^d$ and $\| w_t \| \leq L$ for all $t \in \Nb$.
    Let $W_t = \lambda I + \sum_{s=1}^{t-1} w_s w_s^\top$ for some $\lambda \geq \max(1, L^2)$.
    Then, it holds for any $m \in \Nb$ that
    \begin{equation*}
        \sum_{t=1}^{m} \| w_t \|_{W_t^{-1}}^2 \leq 2 d \log\left(1 + \frac{m}{\lambda d} \right)
    \end{equation*}
\end{lemma}

Then, we give a bound on the number of phases in the following lemma.

\begin{lemma}[Number of Phases]
    \label{lem:num_phases}
    It holds that $N \leq 4 d \log(T)$.
\end{lemma}
\begin{proof}
    Because $\det(\bar{V}_{j+1}) \geq 2 \det(\bar{V}_j)$ for all $j \in [1,N-1]$ by definition, it holds that
    \begin{align*}
        & \det(V_T) \geq \det(\bar{V}_N) \geq 2 \det(\bar{V}_{N-1}) \geq \cdots \geq 2^{N - 1} \det(\bar{V}_1)\\
        & \Longrightarrow \quad 2^{N - 1} \leq \frac{\det(V_T)}{\det(\bar{V}_1)}\\
        & \Longleftrightarrow \quad (N - 1) \log(2) \leq \log\left( \frac{\det(V_T)}{\det(\bar{V}_1)} \right)\\
        & \Longrightarrow \quad N \leq 1 + 2 \log\left( \frac{\det(V_T)}{\det(\bar{V}_1)} \right)
    \end{align*}
    At the same time, it holds that,
    \begin{align*}
        \frac{\det(V_T)}{\det(\bar{V}_1)} & = \det(\bar{V}_1^{-1/2} V_T \bar{V}_1^{-1/2})\\
        & \leq \| \bar{V}_1^{-1/2} V_T \bar{V}_1^{-1/2} \|^d\\
        & = \left\| \bar{V}_1^{-1/2} \left( \sum_{t=1}^{T} x_t x_t^\top + \bar{V}_1 \right) \bar{V}_1^{-1/2} \right\|^d \\
        & = \left\| \frac{1}{\lambda} \sum_{t=1}^{T} x_t x_t^\top + I \right\|^d \\
        & \leq \left( \frac{1}{\lambda}  \sum_{t=1}^{T} \left\| x_t x_t^\top \right\|  + 1 \right)^d\\
        & = \left( \frac{1}{\lambda}  \sum_{t=1}^{T} \left\| x_t \right\|^2  + 1 \right)^d\\
        & \leq \left( \frac{T D^2}{\lambda}    + 1 \right)^d.
    \end{align*}
    Putting everything together,
    \begin{equation*}
        N \leq 1 + 2 \log\left( \frac{\det(V_T)}{\det(\bar{V}_1)} \right) \leq 1 + 2 d \log\left( \frac{T D^2}{\lambda} + 1 \right) \leq 1 + 2 d \log\left( T + 1 \right) \leq 4 d \log\left( T \right)
    \end{equation*}
    where we use $\lambda \geq D^2$ and $T \geq 3$.
\end{proof}

\section{Proof of Proposition \ref{prop:expec_viol} (Expectation Guarantees)}
\label{sec:prf_expec}

In this section, we restate and prove Proposition \ref{prop:expec_viol}.

\begin{proposition}[Restatement of Proposition \ref{prop:expec_viol}]
    \label{prop:expec_viol2}
    Choose $\delta =  \min(1/2,\bmin/(2 S D T))$ and $\kappa = \bmin/T$, and set $\lambda$ and $\beta_t$ the same as in Thereom \ref{thm:main_reg}.
    Then, if $T \geq 3$, it holds that $\Eb[A x_t] \leq b$ for all $t \in [T]$ and
    \begin{equation*}
        \Eb[R_T] \leq \frac{6 D G}{\bmin} \beta_T \sqrt{2 d T \log \left( T \right)} + C_{\Ac} \sqrt{2 d T \log(T)} + 2 D G \sqrt{2 T \log(1/\delta)} + G D + \frac{\bmin G}{S}.
    \end{equation*}
\end{proposition}
\begin{proof}
    Let the event that Theorem \ref{thm:main_reg} holds be denoted by $\Ec$ which occurs with probability at least $1 - 2 \delta$.
    We individually show the safety guarantee and regret guarantee in the following sections.

    \paragraph{Safety guarantee:} Running OSOCO (with $\kappa > 0$) ensures that $A x_t \leq b - \kappa \mathbf{1}$ for all $t \in [T]$ conditioned on $\Ec$.
    Therefore, it holds for all $i \in [n]$ that
    \begin{align*}
        \Eb[a_i^\top x_t] & = \Eb[a_i^\top x_t \Ib\{ \Ec \}] + \Eb[a_i^\top x_t \Ib\{ \Ec^c \}]\\
        & \leq b_i - \kappa + 2 S D \delta \\
        & \leq b_i - \bmin/T + \bmin/T \\
        & = b_i
    \end{align*}
    where we use the trivial constraint bound $a_i^\top x_t \leq SD$ in the first inequality and then the choice of $\kappa$ and $\delta$ in the last step.
    This gives the safety guarantee.

    \paragraph{Regret guarantee:}
    Let $\Yctil = \{ x \in \Xc : A x \leq b -  \mathbf{1} \kappa \}$.
    First, we show that when $\kappa \in [0,\bmin]$, it holds that $\alpha x^\star$ is in $\Yctil$ where $\alpha := \frac{\bmin - \kappa}{\bmin}$.
    To do so, we need that both (a) $\alpha x^\star \in \Xc$ and (b) $a_i^\top (\alpha x^\star) \leq b_i - \kappa$ for all $i \in [n]$.
    Point (a) holds because $\alpha \in [0,1]$, $\mathbf{0} \in \Xc$, $x^\star \in \Xc$ and $\Xc$ is convex, so
    \begin{equation*}
        \alpha x^\star = \alpha x^\star + (1 - \alpha) \mathbf{0} \in \Xc.
    \end{equation*}
    Point (b) holds because, for all $i \in [n]$,
    \begin{equation*}
        a_i^\top (\alpha x^\star) = \alpha a_i^\top x^\star \leq \alpha b_i = \frac{\bmin - \kappa}{\bmin} b_i  \leq \frac{b_i - \kappa}{b_i} b_i = b_i - \kappa,
    \end{equation*}
    where the first inequality uses that $x^\star$ is feasible (i.e. $Ax^\star \leq b$) and the second inequality uses the fact that $\frac{\bmin - \kappa}{\bmin}$ is increasing in $\bmin$ and $\bmin \leq b_i$.

    Then, with $\bar{x}^\star = \argmin_{x \in \Yctil} \sum_{t=1}^{T} f_t(x)$ and $R_T^{\mathrm{osoco}}$ as the (high probability) regret bound of OSOCO when $\kappa > 0$, it holds conditioned on $\Ec$ that
    \begin{align*}
        R_T = \sum_{t=1}^T f_t(x_t) - \sum_{t=1}^T f_t(x^\star) & = \sum_{t=1}^T f_t(x_t) - \sum_{t=1}^T f_t(\bar{x}^\star) + \sum_{t=1}^T f_t(\bar{x}^\star) - \sum_{t=1}^T f_t(x^\star) \\
        & \leq R_T^{\mathrm{osoco}} + \sum_{t=1}^T f_t(\bar{x}^\star) - \sum_{t=1}^T f_t(x^\star)\\
        & \leq R_T^{\mathrm{osoco}} + \sum_{t=1}^T f_t(\alpha x^\star) - \sum_{t=1}^T f_t(x^\star) \\
        & \leq R_T^{\mathrm{osoco}} + G \sum_{t=1}^T \| \alpha x^\star - x^\star \| \\
        & \leq R_T^{\mathrm{osoco}} + (1 - \alpha) G D T \\
        & \leq R_T^{\mathrm{osoco}} + \frac{\kappa}{\bmin} G D T \\
        & \leq R_T^{\mathrm{osoco}} + G D,
    \end{align*}
    where the first inequality uses the regret bound on OSOCO given that it plays within $\Yctil$ and $\bar{x}^\star$ is the optimal action within $\Yctil$, the second inequality is due to the fact that $\alpha x^\star$ is in $\Yctil$ and therefore $\sum_{t=1}^{T} f_t(\bar{x}^\star) = \min_{x \in \Yctil} \sum_{t=1}^{T} f_t(x) \leq \sum_{t=1}^{T} f_t (\alpha x^\star)$, and the last inequality uses the choice of $\kappa$.

    We can then apply this to the expected regret to get that
    \begin{align*}
        \Eb[R_T] & = \Eb[R_T \Ib\{\Ec\}] + \Eb[R_T \Ib\{\Ec^c\}]\\
        & \leq R_T^{\mathrm{osoco}} + G D + 2 GDT\delta\\
        & \leq R_T^{\mathrm{osoco}} + G D + \frac{\bmin G}{S},
    \end{align*}
    where we use the trivial bound on the regret given by
    \begin{equation*}
        R_T = \sum_{t=1}^T f_t(x_t) - \sum_{t=1}^T f_t(x^\star) \leq G \sum_{t=1}^T \| x_t - x^\star\| \leq G D T.
    \end{equation*}

    Finally, we need to account for the choice of $\kappa$ in the regret bound $R_T^{\mathrm{osoco}}$.
    We do so by applying Theorem \ref{thm:main_reg} with $b \leftarrow b - \kappa \mathbf{1}$ to get that
    \begin{align*}
        R_T^{\mathrm{osoco}} & = 4 D G \frac{1}{\bmin - \kappa} \beta_T \sqrt{3 d T \log \left( T \right)} + C_{\Ac} \sqrt{4 d T \log(T)} + 2 D G \sqrt{2 T \log(1/\delta)}\\
        & = 4 D G \frac{1}{\bmin(1 - 1/T)} \beta_T \sqrt{3 d T \log \left( T \right)} + C_{\Ac} \sqrt{4 d T \log(T)} + 2 D G \sqrt{2 T \log(1/\delta)}\\
        & \leq 6 D G \frac{1}{\bmin} \beta_T \sqrt{3 d T \log \left( T \right)} + C_{\Ac} \sqrt{4 d T \log(T)} + 2 D G \sqrt{2 T \log(1/\delta)},
    \end{align*} 
    where we use $T \geq 3$.
\end{proof}

\section{Proof of Propositions \ref{prop:relax_reg} (Improved Efficiency)}
\label{sec:pf_effic}

We first prove the regret guarantees of HedgeDescent in Section \ref{sec:proof_prh} and then give the proof of Proposition \ref{prop:relax_reg} in Section \ref{sec:proof_prr}.

\subsection{Regret of HedgeDescent}
\label{sec:proof_prh}

We give the regret guarantees of HedgeDescent (Algorithm \ref{alg:hedg_desc}) in the following proposition.

\begin{proposition}
    \label{prop:reg_hd2}
    Consider the online optimization setting described in Definition \ref{def:reg} except with action set $\bar{\Xc} = \bigcup_{m \in [M]} \Xc_m$ where each $\Xc_m$ is convex.
    If HedgeDescent (Algorithm \ref{alg:hedg_desc}) is played in this setting with $\zeta_t = \sqrt{4 \log(M)}/G D \sqrt{t}$ and $\eta_t = D/G \sqrt{t}$, then it holds that
    \begin{equation*}
        \sum_{t=1}^{T} \Eb_{m_t \sim p_t} [f_t(x_t(m_t))] - \min_{x \in \bar{\Xc}} \sum_{t=1}^{T} f_t(x)  \leq D G \sqrt{T \log(M)} + 3 D G \sqrt{T}.
    \end{equation*}
\end{proposition}
\begin{proof}
    First, note that we can apply the standard regret guarantees of online gradient descent (e.g. Theorem 3.1 \cite{hazan2016introduction}) to get that the regret of each expert $m \in [M]$ is
    \begin{equation*}
        \sum_{t=1}^{T} f_t (x_t (m)) - \sum_{t=1}^{T} f_t (x^\star_m) \leq 3 D G \sqrt{T},
    \end{equation*}
    where $x^\star_m = \argmin_{x \in \Xc_m} \sum_{t=1}^{T} f_t (x)$.
    Then, we cast the problem of choosing $m$ as an expert advice problem.
    To put this into the standard form, we shift and scale the losses by defining the surrogate losses $\ell_t(m) = \frac{f_t(x_t(m)) - \bar{f}_t}{G D}$ where $\bar{f}_t = \min_{y \in \bar{\Xc}} f_t(y)$.
    This ensures that the $\ell_t(m) \in [0,1]$.
    Then, the Hedge update becomes
    \begin{align*}
        p_{t+1} (m) & \propto p_{t} (m) \exp(-\zeta_t f_t(x_t(m)))\\
        & = p_{t} (m) \exp(-\zeta_t ( G D \ell_t (m) + \bar{f}_t))\\
        & = p_{t} (m) \exp(-\zeta_t \bar{f}_t) \exp(-\zeta_t G D \ell_t (m))\\
        & \propto p_{t} (m) \exp(-(\sqrt{4 \log(M)}/\sqrt{t}) \ell_t (m)).
    \end{align*}
    Note that this update is exactly the form of the classical Hedge algorithm with step size $\bar{\zeta}_t = \sqrt{4 \log(M)}/\sqrt{t}$.
    As such, using the standard analysis of Hedge with time-dependent step size, e.g. \cite{chernov2010prediction}, it holds that
    \begin{equation*}
        \sum_{t=1}^{T} \Eb_{m_t \sim p_t} [\ell_t (m_t)] - \sum_{t=1}^{T} \ell_t (m^\star) \leq \sqrt{T \log(M)}, 
    \end{equation*}
    where $m^\star$ is such that $x^\star \in \Xc_{m^\star}$.
    Putting everything together, we get that
    \begin{align*}
        & \sum_{t=1}^{T} \Eb_{m_t \sim p_t} [f_t(x_t(m_t))] - \sum_{t=1}^{T} f_t(x^\star)\\
        & = \sum_{t=1}^{T} \Eb_{m_t \sim p_t} [f_t(x_t(m_t))] - \sum_{t=1}^{T} f_t(x_t (m^\star)) + \sum_{t=1}^{T} f_t(x_t (m^\star)) - \sum_{t=1}^{T} f_t(x^\star) \\
        & = D G \sum_{t=1}^{T} \Eb_{m_t \sim p_t} [\ell_t(m_t)] - D G \sum_{t=1}^{T} \ell_t(m^\star) + \sum_{t=1}^{T} f_t(x_t (m^\star)) - \sum_{t=1}^{T} f_t(x^\star) \\
        & \leq D G \sqrt{T \log(M)} + 3 D G \sqrt{T},
    \end{align*}
    where we use the fact that $\sum_{t=1}^{T} f_t(x^\star) = \sum_{t=1}^{T} f_t(x^\star_{m^\star})$.
\end{proof}

\subsection{Completing the proof}
\label{sec:proof_prr}
In the following, we restate and prove Proposition \ref{prop:relax_reg}.

\begin{proposition}[Duplicate of Proposition \ref{prop:relax_reg}]
    \label{prop:relax_reg2}
    If OSOCO is modified such that $\Yc_j^o \leftarrow \{ \Yctil_j^o(k,\xi) \}_{k, \xi}$ and HedgeDescent is used for $\Ac$, then the following hold:
    \begin{itemize}
        \item \emph{(Static constraints)} Choosing the algorithm parameters as in Theorem \ref{thm:main_reg} ensures, with probability at least $1 - 2 \delta$, that $A x_t \leq b$ and $R_T \leq \Octil(d^{3/2} \sqrt{T})$.
        \item \emph{(Time-varying constraints)} Choosing the algorithm parameters as in Proposition \ref{prop:expec_viol} ensures that $\Eb[y_t] \leq b$ and $\Eb[R_T] \leq \Octil(d^{3/2} \sqrt{T})$.
    \end{itemize}
\end{proposition}
\begin{proof}
    First note that, following the same process as Lemma \ref{lem:gam_bound}, it holds that $\gamma_t \geq 1 - \frac{1}{\bmin} \sqrt{d} \bar{\beta}_{j_t} \| \bar{V}_{j_t}^{-1/2} x_t \|_{\infty}$.
    Then, using the equivalence of norms, it follows that
    \begin{equation*}
        \gamma_t \geq 1 - \frac{1}{\bmin} \sqrt{d} \bar{\beta}_{j_t} \| \bar{V}_{j_t}^{-1/2} x_t \|_{\infty} \geq 1 - \frac{1}{\bmin} \sqrt{d} \bar{\beta}_{j_t} \|  x_t \|_{\bar{V}_{j_t}^{-1}}
    \end{equation*}
    Then, following the rest of the proof of Theorem \ref{thm:main_reg} gives that, under the conditions of Theorem \ref{thm:main_reg},
    \begin{align*}
        R_T \leq \mbox{ } & 8 D G \frac{1}{\bmin} \beta_T d \sqrt{2 T \log \left( T \right)} + C_{\Ac} \sqrt{2 d T \log(T)} + 2 D G \sqrt{2 T \log(1/\delta)},
    \end{align*}
    and $A x_t \leq b \ \forall t$ with probability at least $1 - 2\delta$.
    Then, following the proof of Proposition \ref{prop:expec_viol}, it holds under the conditions of Proposition \ref{prop:expec_viol} that
    \begin{align*}
        \Eb[R_T] \leq \mbox{ } & 8 D G \frac{1}{\bmin} \beta_T d \sqrt{2 T \log \left( T \right)} + C_{\Ac} \sqrt{2 d T \log(T)} + 2 D G \sqrt{2 T \log(1/\delta)} + G D + \frac{\bmin G}{S},
    \end{align*}
    and $\Eb[y_t] \leq b \ \forall t$.
    Finally, applying the regret guarantees of HedgeDescent by taking $C_\Ac = D G \sqrt{\log(2 d)} + 3 D G$ gives the result.
\end{proof}

\section{Proof of Corollary \ref{cor:tvar} (Stochastic Constraints)}
\label{sec:tvar_apx}

In this section, we restate and prove Corollary \ref{cor:tvar}.

\begin{corollary}
    \label{cor:tvar2}
    Suppose that the cost functions and action set satisfy Assumptions \ref{ass:cost_funcs} and \ref{ass:set_bound}.
    Also, assume that the constraint function satisfies Assumption \ref{ass:stoch}, and let $\rho^2 = (G_g D + 2 F)^2$.
    Then playing OSOCO (Algorithm \ref{alg:main_alg}) with the algorithm parameters chosen as in Proposition \ref{prop:expec_viol} ensures that $\Eb[g_t(x_t)] \leq 0$ for all $t \in [T]$ and $\Eb[R_T^{\mathrm{stoch}}] = \Octil((d + C_{\Ac} \sqrt{d}) \sqrt{T})$.
    Furthermore, with modifications in Proposition \ref{prop:relax_reg}, it holds that $\Eb[R_T^{\mathrm{stoch}}] = \Octil(d^{3/2} \sqrt{T})$.
\end{corollary}
\begin{proof}
    To show the claim, we need to show that the setting of OCO with stochastic constraints is subsumed by OCO with static constraints and no violation in expectation, and the guarantees will follow immediately from Proposition \ref{prop:expec_viol}.
    We do so by taking $g_t(x_t) = y_t$, $A := \Eb[A_t]$, $b := \Eb[b_t]$ and $\epsilon_t = (A_t - A) x_t - b_t + b$.

    First, let $\Fc_t = \sigma(x_1, \epsilon_1, ..., \epsilon_{t-1}, x_t)$.
    Also, we restate Assumption \ref{ass:stoch} in terms of $A_t$ and $b_t$:
    \begin{enumerate}[(i)]
        \item $\| a_{t,i} \| \leq G_g$ for all $i \in [n]$,
        \item $\| b_t \| \leq F$, 
        \item $b > 0$,
        % \item $A_t$ and $b_t$ are bounded for all $t \in [T]$,\label{item:bounded}
        \item $b$ is known.
    \end{enumerate}
    First, we show how the constraint $g_t(x_t) \leq 0$ is subsumed by the model in Section \ref{sec:prob_set}.
    Indeed,
    \begin{equation*}
        y_t = g_t (x_t) = A_t x_t - b_t = A x_t - b + \underbrace{(A_t - A) x_t - b_t + b}_{\epsilon_t}.
    \end{equation*}
    It remains to show that Assumption \ref{ass:noise} holds, i.e. that $\Eb[\epsilon_{t,i} | \Fc_t] = 0$ and $\Eb[\exp(\lambda \epsilon_{t,i}) | \Fc_t] \leq \exp(\frac{\lambda^2 \rho^2}{2})\ \forall \lambda \in \Rb$.
    Indeed,
    \begin{equation*}
        \Eb[\epsilon_{t,i} | \Fc_t] = \Eb[(a_{t,i} - a_{i})^\top x_t | \Fc_t] - \Eb[b_{t,i} | \Fc_t] + b_i = \Eb[(a_{t,i} - a_{i})^\top | \Fc_t] x_t  - \Eb[b_{t,i}] + b_i = \Eb[(a_{t,i} - a_{i})^\top] x_t = 0,
    \end{equation*}
    where the second equality uses that $x_t$ is $\Fc_t$-measurable and that $b_t$ is independent of $\Fc_t$, and the third equality uses that $A_t$ is independent of $\Fc_t$.
    Also, it holds almost surely that,
    \begin{align*}
        | \epsilon_{t,i} | & = | (a_{t,i} - a_{i})^\top x_t  - b_{t,i} + b_i|\\
        & \leq (\| a_{t,i}\| + \| a_{i} \| ) \| x_t \| + |b_{t,i}| + |b_i| \\
        & = (\| a_{t,i} \| + \| \Eb[a_{t,i}] \| ) \| x_t \| + |b_{t,i}| + |\Eb[b_{t,i}]| \\
        & \leq (\| a_{t,i} \| + \Eb[\| a_{t,i} \| ] ) \| x_t \| + |b_{t,i}| + \Eb[|b_{t,i}|] \\
        & \leq G_g D + 2 F.
    \end{align*}
    Therefore, it follows from the Hoeffding Lemma that,
    \begin{equation*}
        \Eb[\exp(\lambda \epsilon_{t,i}) | \Fc_t] \leq \exp\left( \frac{\lambda^2 (G_g D + 2 F)^2}{2} \right),
    \end{equation*}
    and therefore we can take $\rho^2 = ( G_g D + 2 F)^2$.
    Therefore, the constraint $g_t(x_t) \leq 0$ is subsumed by the model in Section \ref{sec:prob_set}.

    Finally, note that,
    \begin{equation*}
        \Eb[g_t(x_t)] = \Eb[A x_t] - b + \Eb[\epsilon_t] = \Eb[A x_t] - b + \Eb[\Eb[\epsilon_t | \Fc_t]] = \Eb[A x_t] - b.
    \end{equation*}
    Therefore, the in-expectation guarantees in Proposition \ref{prop:expec_viol}, $\Eb[A x_t] \leq b$ imply that $\Eb[g_t(x_t)] \leq 0$.
    At the same time, it holds that $\bar{g}(x) = \Eb[A_t] x - \Eb[b_t] = A x - b$ and therefore the regret notions are identical, i.e. $R_T = R_T^\mathrm{stoch}$.
\end{proof}

\section{Example OCO Algorithms}
\label{sec:examp_algs}

In this section, we give examples of algorithms that can be used for the OCO algorithm $\Ac$ in OSOCO, i.e. those that admit the type of regret bound specified in Definition \ref{def:reg}.

\subsection{Problem setting}

We consider a setting where at each round $t$, the player chooses an action $x_t \in \Xc \subseteq \Rb^d$ and an adversary simultaneously chooses a cost function $f_t$ given the actions played in previous rounds.
The player then observes $f_t$ and suffers the cost $f_t(x_t)$.
The game can be terminated at any time and therefore the final round $T$ is unknown to the player.
In the algorithms below, we use the flag $\mathsf{alive}$ to indicate that the game has not been terminated.
Also, $\Xc$ is not necessarily convex.
We assume that all $f_t$ are differentiable and $\| \nabla f_t (x) \| \leq G$ for all $x \in \Xc$.
We also assume that $\| x \| \leq D/2$ for all $x \in \Xc$.

\subsection{Example Algorithm 1: Anytime Hedge over a Finite Cover}
\label{sec:hedge_cover}

The simplest algorithm to handle this problem is Hedge played over a finite cover with the doubling trick.
We give pseudocode for this algorithm in Algorithm \ref{alg:hedge}.
This is a well-known extension of the Hedge algorithm \cite{freund1997decision} that has been noted by works such as \cite{dani2007price}, \cite{maillard2010online} and \cite{krichene2015hedge}.

\begin{algorithm}[h]
    \caption{Anytime Hedge over a Finite Cover}
    \label{alg:hedge}
\begin{algorithmic}[1]
    \INPUT $\Xc$.
    \STATE Initialize: $t = 1$, $j = 0$, $\mathsf{alive} = \mathtt{True}$.
    \WHILE{$\mathsf{alive}$}
        \STATE Construct $\Delta_{j}$-net over $\Xc$, i.e. choose points $\{z_1,...,z_{M_{j}}\}$ such that $\min_{m \in [M_{j}]} \| x - z_m \| \leq \Delta_{j}$ for all $x \in \Xc$.
        \STATE Set distribution $p_t$ over $[M_{j}]$ to be uniform.
        \WHILE{$t \leq \tau_{j+1} - 1$ \textbf{ and } $\mathsf{alive}$}
            \STATE Sample $m_t \sim p_t$.
            \STATE Play $x_t = z_{m_t}$ and observe $f_t$.
            \STATE Update $p_{t + 1}(m) \propto \exp(-\eta_{j} f_t (z_m))$.
            \STATE $t = t + 1$.
        \ENDWHILE
    \STATE $j = j + 1$
    \ENDWHILE
\end{algorithmic}
\end{algorithm}

Then, we give the regret guarantees of Algorithm \ref{alg:hedge}.

\begin{proposition}
    Consider Algorithm \ref{alg:hedge} with $\Delta_j = 1 / (G \sqrt{\tau_j})$ and $\eta_j = \sqrt{2 \log(M_j)/\tau_j}/GD$.
    Then, it holds that
    \begin{equation*}
        \sum_{t=1}^{T} \Eb_{x_t \sim p_t} [f_t(x_t)] - \min_{x \in \Xc} \sum_{t=1}^{T} f_t(x) \leq 5 \left(G D \sqrt{4 d\log(2 D G \sqrt{d T})} + 2 \right) \sqrt{T}
    \end{equation*}
\end{proposition}
\begin{proof}
    Let $x^\star$ be the optimal action over the entire time horizon, $x^\star_j$ be the optimal action within phase $j$ and $m^\star_j$ be the index of the nearest net point.
    Also, let $\tau_j$ be the duration of phase $j$ and $j = N$ be the index of the last phase.
    To put the update into the standard form for Hedge, we shift and scale the losses by defining the surrogate losses $\ell_t(m) = \frac{f_t (z_m) - \bar{f}_t}{G D}$ where $\bar{f}_t = \min_{y \in \Xc} f_t(y)$.
    This ensures that the $\ell_t(m) \in [0,1]$.
    Then, the Hedge update becomes
    \begin{align*}
        p_{t+1} (m) & \propto p_{t} (m) \exp(-\eta_{j} f_t (z_m))\\
        & = p_{t} (m) \exp(-\eta_{j} ( G D \ell_t (m) + \bar{f}_t))\\
        & = p_{t} (m) \exp(-\eta_{j} \bar{f}_t) \exp(-\eta_{j} G D \ell_t (m)) \\
        & \propto p_{t} (m) \exp(-\eta_{j} G D \ell_t (m)) \\        
        & = p_{t} (m) \exp(-(\sqrt{2 \log(M_j)/\tau_j}) \ell_t (m)).
    \end{align*}
    This is equivalent to the classical analysis of Hedge \cite{arora2012multiplicative} with step size $\eta_j' = \sqrt{2 \log(M_j)/\tau_j}$ and therefore, for any phase $j \in [0,N]$,
    \begin{equation}
        \label{eqn:hedge_reg}
        \begin{split}
            & \sum_{t=\tau_{j}}^{\tau_{j+1} - 1} \Eb_{x_t \sim p_t} [f_t(x_t)] - \min_{m \in M_j} \sum_{t=\tau_j}^{\tau_{j+1} - 1} f_t(z_m)\\
            & = GD \sum_{t=\tau_{j}}^{\tau_{j+1} - 1} \Eb_{x_t \sim p_t} [\ell_t(x_t)] - G D \min_{m \in M_j} \sum_{t=\tau_j}^{\tau_{j+1} - 1} \ell_t(z_m) \\
            & \leq G D \sqrt{4 \tau_j \log(M_j)},
        \end{split}
    \end{equation}
    Note that if the game is terminated in the midst of the final round, there will not be cost functions for the remaining rounds.
    However, we can define the ``fictitious'' cost functions $f_{T+1},...,f_{\tau_{N+1}-1}$ that extend from the end of the game to the end of the final phase.
    This is why we have the regret guarantee for all phases as defined in \eqref{eqn:hedge_reg}.

    Also, for any $x \in \Xc$, there exists $z \in \{z_1,...,z_{M_{j}}\}$ such that
    \begin{equation*}
        f_t(x) - f_t(z) \leq G \| x - z \| \leq G \Delta_j
    \end{equation*}
    Then, we get that within a given phase $j$, it holds that
    \begin{align*}
        & \sum_{t=\tau_{j}}^{\tau_{j+1} - 1} \Eb_{x_t \sim p_t} [f_t(x_t)] - \sum_{t=\tau_j}^{\tau_{j+1} - 1} f_t(x^\star_j)\\
         & \leq \sum_{t=\tau_{j}}^{\tau_{j+1} - 1} \Eb_{x_t \sim p_t} [f_t(x_t)] - \sum_{t=\tau_j}^{\tau_{j+1} - 1} f_t(z_{m^\star_j}) + \sum_{t=\tau_{j}}^{\tau_{j+1} - 1} f_t(z_{m^\star_j}) - \sum_{t=\tau_j}^{\tau_{j+1} - 1} f_t(x^\star_j)\\
        & \leq G D \sqrt{4 \tau_j \log(M_j)} + G \Delta_j \tau_j
    \end{align*}
    Furthermore, it is known (e.g. Example 27.1 in \cite{shalev2014understanding}) that we can construct a grid for the net to get $M_j \leq (2 D \sqrt{d} /\Delta_j)^d$ for desired $\Delta_j$.
    Therefore, choosing $\Delta_j = 1 / (G \sqrt{\tau_j})$ will need $M_j \leq ( 2 D G \sqrt{d \tau_j} )^d$ points.
    Also, note that $N$ is the smallest integer such that $\sum_{j=0}^{N} 2^j = 2^{N+1} - 1 \geq T$ and therefore $N = \lceil \log_2 (T + 1) - 1 \rceil \leq \log_2 (T + 1)$.
    Therefore, it holds that 
    \begin{equation}
        \label{eqn:tot_sqrt}
        \sum_{j=0}^{N} \sqrt{\tau_j} = \sum_{j=0}^{N} \sqrt{2^j} = \sum_{j=0}^{N} \sqrt{2}^j = \frac{1- \sqrt{2}^{N + 1}}{1 - \sqrt{2}} = \frac{\sqrt{2^{N + 1}} - 1}{\sqrt{2} - 1} \leq 5 \sqrt{T}
    \end{equation}
    Putting everything together, we get that
    \begin{align*}
        \sum_{t=1}^{T} \Eb_{x_t \sim p_t} [f_t(x_t)] - \sum_{t=1}^{T} f_t(x^\star) & = \sum_{j=0}^{N} \sum_{t=\tau_j}^{\tau_{j+1} - 1} \Eb_{x_t \sim p_t} [f_t(x_t)] - \sum_{j=0}^{N} \sum_{t=\tau_j}^{\tau_{j+1} - 1} f_t(x^\star)\\
        & \leq \sum_{j=0}^{N} \sum_{t=\tau_j}^{\tau_{j+1} - 1} \Eb_{x_t \sim p_t} [f_t(x_t)] - \sum_{j=0}^{N} \sum_{t=\tau_j}^{\tau_{j+1} - 1} f_t(x^\star_j)\\
        & \leq \sum_{j=0}^{N} (G D \sqrt{4 \tau_j \log(M_j)} + 2 G \Delta_j \tau_j) \\
        & = \sum_{j=0}^{N} \left(G D \sqrt{4 d \tau_j \log(2 D G \sqrt{d \tau_j})} + 2 \sqrt{\tau_j}\right)\\
        & \leq \left(G D \sqrt{4 d\log(2 D G \sqrt{d T})} + 2 \right) \sum_{j=0}^{N} \sqrt{\tau_j}\\
        & \leq 5 \left(G D \sqrt{4 d\log(2 D G \sqrt{d T})} + 2 \right) \sqrt{T}
    \end{align*}
    completing the proof.
\end{proof}

\subsection{Example Algorithm 2: Anytime Follow the Perturbed Leader}

Another algorithm that can be used for this setting is Follow the Perturbed Leader (FTPL) \cite{kalai2005efficient} which has been popular for a variety of OCO settings.
We use the version from \cite{suggala2020online} with an exact optimization oracle.
We use the version and analysis from \cite{suggala2020online} because they study it in the nonconvex setting, which is directly applicable to our setting because the action set is nonconvex.
We give the FTPL algorithm with the doubling trick in Algorithm \ref{alg:ftpl}.

\begin{algorithm}[h]
    \caption{Anytime Follow the Perturbed Leader}
    \label{alg:ftpl}
\begin{algorithmic}[1]
    \INPUT $\Xc$.
    \STATE Initialize: $t = 1$, $j = 0$, $\mathsf{alive} = \mathtt{True}$, $\tau_j := 2^j$.
    \WHILE{$\mathsf{alive}$}
        \WHILE{$t \leq \tau_{j+1} - 1$ \textbf{ and } $\mathsf{alive}$}
            \STATE Let $\sigma_t = [\sigma_{t,1}\ ...\ \sigma_{t,d}]^\top$ where $\sigma_{t,i} \sim \mathrm{Exp}(\eta_j)$.
            \STATE Choose $x_t \in \argmin_{x \in \Xc}\left( \sum_{s=1}^{t-1} f_t(x) - \sigma_t^\top x \right)$
            \STATE Play $x_t$ and observe $f_t$.
            \STATE $t = t + 1$.
        \ENDWHILE
    \STATE $j = j + 1$
    \ENDWHILE
\end{algorithmic}
\end{algorithm}

We then give the regret guarantees of this algorithm in the following.

\begin{proposition}
    Consider Algorithm \ref{alg:ftpl} with $\eta_j = 1/\sqrt{2 d G^2 \tau_j}$.
    Then, it holds that
    \begin{equation*}
        \sum_{t=1}^{T} \Eb_{x_t \sim p_t} [f_t(x_t)] - \min_{x \in \Xc} \sum_{t=1}^{T} f_t(x) \leq 630 D G d^{3/2} \sqrt{2 T},
    \end{equation*}
    where $p_t$ is the distribution over the action round $t$.
\end{proposition}
\begin{proof}
    We will use the analysis from \cite{suggala2020online} to bound the regret within each phase.
    First note that \cite{suggala2020online} assumes that $\| x - y \|_{\infty} \leq D'$ for all $x,y \in \Xc$ and that $f_t$ are $L'$-Lipschitz w.r.t. the 1-norm.
    In our setting, it holds that for all $x,y \in \Xc$ that 
    \begin{equation*}
        \| x - y \|_\infty \leq \| x - y \| = \| x - \mathbf{0} + \mathbf{0} - y \| \leq \| x - \mathbf{0} \| + \| \mathbf{0} - y \| \leq D
    \end{equation*}
    and that
    \begin{equation*}
        | f_t (x) - f_t (y) | \leq G \| x - y \| \leq G \| x - y \|_1.
    \end{equation*}
    Therefore, we can take $D' = D$ and $L' = G$.
    Also, \cite{suggala2020online} assume access to a $(\alpha,\beta)$ inexact oracle, but we assume that this oracle is exact and therefore take $\alpha,\beta = 0$.
    Then, since the regret guarantees of FTPL hold against a non-oblivious adversary (see Theorem 8 in \cite{suggala2020online}), we can apply the regret guarantees of FTPL (Theorem 1 in \cite{suggala2020online}) to get that,
    \begin{align*}
        \sum_{t=\tau_{j}}^{\tau_{j+1} - 1} \Eb_{x_t \sim p_t} [f_t(x_t)] - \min_{x \in \Xc} \sum_{t=\tau_j}^{\tau_{j+1} - 1} f_t(x) & \leq 125 \eta_j (L')^2 d^2 D' 2 \tau_j + \frac{d D'}{\eta_j}\\
        & = 125 \eta_j G^2 d^2 D 2 \tau_j + \frac{d D}{\eta_j}\\
        & = 126 d^{3/2} D G \sqrt{2 \tau_j}
    \end{align*}
    where we use $\eta_j = 1/\sqrt{2 d G^2 \tau_j}$.
    Then, putting everything together and letting $x^\star = \argmin_{x \in \Xc} \sum_{t=1}^{T} f_t(x)$, we get that
    \begin{align*}
        \sum_{t=1}^{T} \Eb_{x_t \sim p_t}[f_t(x_t)] - \sum_{t=1}^{T} f_t(x^\star) & = \sum_{j=0}^{N} \sum_{t=\tau_j}^{\tau_{j+1} - 1} \Eb_{x_t \sim p_t}[f_t(x_t)] - \sum_{j=0}^{N} \sum_{t=\tau_j}^{\tau_{j+1} - 1} f_t(x^\star)\\
        & \leq \sum_{j=0}^{N-1} \sum_{t=\tau_j}^{\tau_{j+1} - 1} \Eb_{x_t \sim p_t} [f_t(x_t)] - \sum_{j=0}^{N} \sum_{t=\tau_j}^{\tau_{j+1} - 1} f_t(x^\star_j)\\
        & \leq \sum_{j=0}^{N} 126 D G d^{3/2} \sqrt{2 \tau_j} \\
        & \leq 630 D G d^{3/2} \sqrt{2 T},
    \end{align*}
    where we use \eqref{eqn:tot_sqrt} in the last inequality.
\end{proof}

\section{Details on Numerical Experiments}
\label{sec:det_num}

In this appendix, we provide details on the numerical experiments from Section \ref{sec:num_exp}.
We consider a $2$-dimensional setting with quadratic costs of the form $f_t(x) = 3 \| x - v_t \|^2$.
We choose $\Xc = \Bb$, $A = [-1\ -1]$, $b = 0.8$ and sample $v_t \sim [-1,0]^d$ and $\epsilon_t \sim \Nc(0,\sigma I)$ where $\sigma = 0.01$.
In this setting, we consider OSOCO configured for high-probability (OSOCO-H) and OSOCO configured for static expectation (OSOCO-E) alongside the algorithm from \cite{yu2017online} (which we call DPP for "drift-plus-penalty") and SO-PGD from \cite{chaudhary2022safe}.
For DPP, we consider both the nominal choice of algorithm parameters as well as a conservative choice of algorithm parameters that are chosen to keep the violation very small (referred to as DPP(cons.)).
The choice of algorithm parameters for all algorithms is given in Table \ref{tab:algorithm_parameters}.

We study the regret, the stochastic constraint violation and the static constraint violation as reported in Figure \ref{fig:expers:a}, Figure \ref{fig:expers:b} and Figure \ref{fig:expers:c}, respectively.
All figures show the results for $30$ trials, where the average is indicated by a point and the standard deviation is indicated by an errorbar or shading.
These results indicate that the nominal DPP algorithm has lower regret than any other algorithm, although it also incurs significant violation of both static and stochastic constraints.
When the algorithm parameters for DPP are chosen to ensure very little violation, the regret becomes substantial.
The algorithms with the next lowest regret are OSOCO-T and OSOCO-S, which do not incur any stochastic constraint violation (on average) or static constraint violation.

Next, we give the details on the setting, and then give the algorithm parameters for each of the algorithms in Table \ref{tab:algorithm_parameters}.
Note that, unless otherwise specified, these are chosen as recommended in each particular paper.
In total, the experiments took about $30$ hours of CPU compute time.

\paragraph{Setting Details}
\begin{itemize}
    \item $f_t(x) = 3 \| x - v_t \|^2$
    \item $g(x) = A x - b$ (Static Constraint)
    \item $g_t(x) = A x - b_t$ (Stochastic Constraint)
    \item $d = 2$
    \item $\Xc = \Bb$
    \item $b_t = b + \epsilon_t$
    \item $A = [-1 -1]$
    \item $v_t \sim [-1,0]^d$
    \item $\epsilon_t \sim \Nc(0,\rho I)$
    \item $\rho = 0.01$
    \item $b = 0.8$
    \item $D = 2$
    \item $S = \sqrt{2}$
    \item $G = 6\sqrt{2}+6$
    \item $\bmin = 0.8$
\end{itemize}

\begin{table}[h]
    \centering
    \caption{Choice of algorithm parameters for simulations.}
    \label{tab:algorithm_parameters}
    \renewcommand{\arraystretch}{1.2}  % Increase row height slightly
    \begin{tabular}{|c|c|l|}
    \hlineB{3}
    \rowcolor{gray!20}  % Light gray background for header
    \textbf{ Algorithm} & \textbf{ Version} & \textbf{ Parameters} \\
    \hlineB{2}
    \multirow{14}{*}[-3ex]{OSOCO} & \multirow{7}{*}[-1.5ex]{E} & $\Ac = \mathrm{HedgeDescent}$ \\
     & & $\eta_t = D/(G\sqrt{t})$ (in HedgeDescent) \\
     & & $\zeta_t = \sqrt{4\log(2 d)}/(G D \sqrt{t})$ (in HedgeDescent) \\
     & & $\lambda = \max(1,D^2)$ \\
     & & $\delta = \min(1/2,\bmin/(2 S D T))$ \\
     & & $\beta_t = \rho \sqrt{ d \log \left(\frac{1 + (t - 1) D^2/\lambda}{\delta/n} \right)} + \sqrt{\lambda} S$ \\
     & & $\kappa = \bmin/T$ \\
    \cline{2-3}
     & \multirow{7}{*}[-1.5ex]{H} & $\Ac = \mathrm{HedgeDescent}$ \\
     & & $\eta_t = D/(G\sqrt{t})$ (in HedgeDescent) \\
     & & $\zeta_t = \sqrt{4\log(2 d)}/(G D \sqrt{t})$ (in HedgeDescent) \\
     & & $\lambda = \max(1,D^2)$ \\
     & & $\delta = 0.01$ \\
     & & $\beta_t = \rho \sqrt{ d \log \left(\frac{1 + (t - 1) D^2/\lambda}{\delta/n} \right)} + \sqrt{\lambda} S$ \\
     & & $\kappa = 0$ \\
    \hline
    \multirow{4}{*}[-0.5ex]{DPP} & \multirow{2}{*}{-} & $\alpha = T$ \\
     & & $V = \sqrt{T}$ \\
    \cline{2-3}
     & \multirow{2}{*}{(cons.)} & $\alpha = T$ \\
     & & $V = 0.01\sqrt{T}$ \\
    \hline
    \multirow{4}{*}[-0.5ex]{SO-PGD} & \multirow{4}{*}[-0.5ex]{-} & $\lambda = 1$ \\
     & & $\delta = 0.01$ \\
     & & $\eta = D/(G\sqrt{T})$ \\
     & & $T_0 = T^{2/3}$ \\
    \hlineB{3}
    \end{tabular}
\end{table}

\end{document}